\documentclass{article}

\usepackage[english]{babel}

\usepackage{graphicx}
\usepackage{microtype}
\usepackage{booktabs}

\usepackage{float}
\usepackage{wrapfig}
\usepackage[normalem]{ulem}

\usepackage{bm}
\usepackage{textcomp}
\usepackage{amssymb}
\usepackage{amsmath}
\usepackage{amsthm}
\usepackage{amsfonts}
\usepackage{mathtools}
\usepackage{pifont}

\usepackage{algorithm}
\usepackage{algorithmic}

\usepackage{times}
\usepackage{gensymb}

\usepackage{textcomp}
\usepackage{multirow}
\usepackage{lscape}
\usepackage{subcaption}

\usepackage{mdframed}
\usepackage{hyperref}
\usepackage{textgreek}
\usepackage{tabularx, colortbl}
\usepackage[dvipsnames]{xcolor}
\definecolor{mygray}{gray}{0.95}
\definecolor{mycyan}{HTML}{005397}
\definecolor{myred}{HTML}{E13333}
\definecolor{mymagenta}{HTML}{BF3E87}
\definecolor{mypurple}{HTML}{1B2278}

\definecolor{tearose}{HTML}{F584C5}
\definecolor{coral}{HTML}{F67088}
\definecolor{dodger_blue}{HTML}{3BA3EC}
\definecolor{domino}{HTML}{BC9F48}
\definecolor{domino}{HTML}{BC9F48}
\definecolor{domino}{HTML}{BC9F48}
\definecolor{catalina_blue}{HTML}{1C3168}
\definecolor{catalina_blue}{HTML}{1C3168}
\definecolor{catalina_blue}{HTML}{1C3168}
\definecolor{dark_scarlet}{HTML}{C63D52}
\definecolor{cerulean}{HTML}{0192A8}

\definecolor{tussock}{HTML}{C99E31}
\definecolor{p13}{HTML}{BFB5D7}
\definecolor{b14}{HTML}{BEA1A5}
\definecolor{y15}{HTML}{F0Cf61}
\definecolor{Merino}{HTML}{F3EEE3}

\newcolumntype{a}{>{\columncolor{p13}}l}

\usepackage[capitalize,noabbrev]{cleveref}
\crefname{ineq}{Inequality}{Inequalities}
\creflabelformat{ineq}{#2{\upshape(#1)}#3}

\makeatletter
\newtheorem*{rep@theorem}{\rep@title}
\newcommand{\newreptheorem}[2]{%
\newenvironment{rep#1}[1]{%
 \def\rep@title{#2 \ref{##1}}%
 \begin{rep@theorem}}%
 {\end{rep@theorem}}}
\makeatother

\theoremstyle{remark}
\newtheorem{example}{Example}[]
\newreptheorem{example}{Example}

\theoremstyle{plain}
\newtheorem{theorem}{Theorem}[section]
\newreptheorem{theorem}{Theorem}

\newtheorem{lemma}[theorem]{Lemma}
\newreptheorem{lemma}{Lemma}

\newtheorem{corollary}[theorem]{Corollary}
\newreptheorem{corollary}{Corollary}

\theoremstyle{definition}
\newtheorem{definition}[theorem]{Definition}

\theoremstyle{remark}
\newtheorem{remark}[theorem]{Remark}

\usepackage{braket}
\newcommand{\norm}[1]{\left\lVert#1\right\rVert}
\newcommand{\abs}[1]{\left\lvert#1\right\rvert}

\usepackage{stmaryrd}

\DeclareMathOperator*{\argmax}{arg\,max}
\DeclareMathOperator*{\argmin}{arg\,min}

\usepackage{tikz}
\usetikzlibrary{trees}
\usepackage{tikz-dependency}
\usepackage{tikz-3dplot}
\usetikzlibrary{chains,scopes}
\usetikzlibrary{arrows.meta,automata,positioning}
\usetikzlibrary{shapes.geometric, arrows}
\usepackage{pgfplots}
\pgfplotsset{compat=newest}
\usepgfplotslibrary{fillbetween}

\pgfplotsset{
  every axis/.append style = {thick},
  tick style = {thick,black},
  /tikz/normal shift/.code 2 args = {%
    \pgftransformshift{%
        \pgfpointscale{#2}{\pgfplotspointouternormalvectorofticklabelaxis{#1}}%
    }%
  },%
  shift/.style = {
    tick align        = outside,
    scaled ticks      = false,
    enlargelimits     = false,
    ticklabel shift   = {#1},
    axis lines*       = left,
    xtick style       = {normal shift={x}{#1}},
    ytick style       = {normal shift={y}{#1}},
    x axis line style = {normal shift={x}{#1}},
    y axis line style = {normal shift={y}{#1}},
  },
  shift/.default = 10pt,
  shift3d/.style = {
    shift=#1,
    ztick style       = {normal shift={z}{#1}},
    z axis line style = {normal shift={z}{#1}},
  },
  shift3d/.default = 10pt,
}

\usepackage{listings}
\usepackage{array}
\newcolumntype{H}{>{\setbox0=\hbox\bgroup}c<{\egroup}@{}}

\usepackage[preprint,nonatbib]{neurips_2024}

\usepackage[utf8]{inputenc} 
\usepackage[T1]{fontenc}    
\usepackage{hyperref}       
\usepackage{url}            
\usepackage{booktabs}       
\usepackage{amsfonts}       
\usepackage{nicefrac}       
\usepackage{microtype}      
\usepackage{xcolor}         

\title{How Does Bayes Error Limit Probabilistic Robust Accuracy}

\author{%
  \hfill Ruihan Zhang\hfill\hfill Jun Sun\hfill\hfill\\
  School of Computing and Information Systems\\
  Singapore Management University\\
  \texttt{\{rhzhang,junsun\}@smu.edu.sg} \\
}

\begin{document}

\maketitle

\begin{abstract}
Adversarial examples pose a security threat to many critical systems built on neural networks. Given that deterministic robustness often comes with significantly reduced accuracy, probabilistic robustness (\emph{i.e.}, the probability of having the same label with a vicinity is $\ge 1-\kappa$) has been proposed as a promising way of achieving robustness whilst maintaining accuracy. However, existing training methods for probabilistic robustness still experience non-trivial accuracy loss. It is unclear whether there is an upper bound on the accuracy when optimising towards probabilistic robustness, and whether there is a certain relationship between $\kappa$ and this bound. This work studies these problems from a Bayes error perspective. We find that while Bayes uncertainty does affect probabilistic robustness, its impact is smaller than that on deterministic robustness. This reduced Bayes uncertainty allows a higher upper bound on probabilistic robust accuracy than that on deterministic robust accuracy. Further, we prove that with optimal probabilistic robustness, each probabilistically robust input is also deterministically robust in a smaller vicinity. We also show that voting within the vicinity always improves probabilistic robust accuracy and the upper bound of probabilistic robust accuracy monotonically increases as $\kappa$ grows. Our empirical findings also align with our results.
\end{abstract}

\section{Introduction}
\label{sec:intro}

Neural networks (NNs) achieve remarkable success in various applications, including many security-critical systems~\cite{kurakin2018adversarial,sharif2016accessorize}. At the same time, several security vulnerabilities in NNs have been identified, including adversarial attacks that generate adversarial examples. Adversarial examples are inputs that are carefully crafted by adding human imperceptible perturbation to normal inputs to trigger wrong predictions~\cite{kurakin2016adversarial}. Their presence is particularly concerning in security-critical NN applications.

To defend against adversarial examples, various methods for improving a model's robustness have been proposed. Adversarial training works by training NNs with a mix of normal and adversarial examples, either pre-generated or generated during training. As it does not carry a formal guarantee on the achieved robustness~\cite{zhang2019limitations}, adversarially trained NNs are potentially vulnerable to new types of adversarial attacks~\cite{liu2019adaptiveface,tramer2020adaptive}. In contrast, certified training aims to provide a formal guarantee of robustness. A method in this category typically incorporates robustness verification techniques during training~\cite{xu2020automatic}, \emph{i.e.}, they aim to find a valuation of network parameters such that the model is provably robust to the training samples and some definition of vicinity. However, they are often impractical for multiple reasons including the irreducible errors due to Bayes error (the unavoidable inaccuracy in collecting or labelling training samples) which limits the level of accuracy that is achievable~\cite{chiang2020certified,zhang2024certified}.

Recent studies propose that probabilistic robustness, \emph{i.e.}, the probability of having adversarial examples within a vicinity is no more than a certain tolerance level $\kappa$, could be sufficient for many practical applications~\cite{robey2022probabilistically,li2022towards,zhang2023proa}. Furthermore, it is shown to be achievable with a smaller accuracy drop ($\sim$5\% when $\kappa = 0.1$) and less computational cost compared to certified training methods~\cite{robey2022probabilistically}. Probabilistic robustness thus balances between achieving high degrees of security and maintaining accuracy. However, it remains unclear if Bayes errors also limit the achievable performance when optimising towards probabilistic robustness \emph{i.e.}, probabilistic robust accuracy. Further, if an upper limit does exist, how is such a limit correlated to the tolerance level $\kappa$?

In this work, we aim to answer these questions. The Bayes error, in the context of statistics and machine learning, is a fundamental concept related to the inherent uncertainty in any classification system~\cite{ishida2022performance}. It represents the minimum error for any classifier on a given problem and is determined by the overlap in the probability distributions of different classes~\cite{fukunaga1990introduction}. We remark that the relevance of Bayes error in simple classification tasks may occasionally be questioned given that many datasets, such as MNIST, provide a single, definite label for each input~\cite{lecunmnist}. However, real-world data often lacks this clarity due to inevitable information loss, \emph{e.g.}, during image capture or compression. For instance, the CIFAR-10H dataset showcases that over a third of CIFAR-10 inputs can be re-annotated with uncertain labels by human annotators (CIFAR-10H)~\cite{peterson2019human}. Thus, this uncertainty leads to Bayes errors, which fundamentally constrain not only vanilla accuracy~\cite{ishida2022performance} but also deterministic robust accuracy~\cite{zhang2024certified} and probabilistic robust accuracy (as what will be shown in this work).


We study the limit on the probabilistic robust accuracy resulting from Bayes error. We first derive an optimal decision rule that maximises probabilistic robust accuracy. Similar to the Bayes classifier~\cite{fukunaga1975k}, the optimal decision rule for probabilistic robustness is also a Maximum A Posteriori (MAP~\cite{bassett2019maximum}) probability decision, except that the posterior is regarding the vicinity, not a single input. Then, we show that the error from this optimal decision rule regarding probabilistic robustness is lower bounded by the Bayes error for deterministic robustness, but within a much smaller vicinity. After that, a relationship is established between the upper bound of probabilistic robust accuracy and the upper bound of vanilla accuracy or deterministic robust accuracy. We further show that the bound monotonically increases as $\kappa$ grows. Empirically, we show that our bounds are consistent with what is observed on those probabilistically robust neural networks trained on various distributions.

\section{Preliminary and Problem Definition}
\label{sec:background}

This section first reviews the background of robustness in machine learning. Then, we recall the Bayes error for the deterministic robustness of classification. Finally, we define our research problem.

\subsection{Robustness in Neural Network Classification}

We put the context in a $K$-class classification problem where a classifier $h:\bm{x}\mapsto y$ learns to fit a joint distribution $D$ over input space $\mathbb{R}^n$ and label space $\set{0, 1, {\scriptstyle\ldots}, K-1}$. Let $h(\bm{x}) \in \set{0, 1, {\scriptstyle\ldots}, K-1}$ denote prediction, and an error captures the difference between $h(\bm{x})$ and $y$. That is, vanilla accuracy is $\Upsilon^+_\text{acc}(D, h) = \operatorname{E}_{(\mathbf{x}, \textnormal{y}) \sim D} \left[\mathbf{1}_{h(\mathbf{x})=\textnormal{y}}\right]$ and thus its error is $\Upsilon^-_\text{acc}(D, h) = 1 - \Upsilon^+_\text{acc}(D, h)$. The capital Upsilon with a plus sign denotes accuracy itself, while a minus sign denotes its corresponding error. 

Robustness $\Upsilon^+_\text{rob}(D, h,\mathbb{V})$ measures the change in prediction when a perturbation occurs on the input~\cite{szegedy2013intriguing}. If the prediction changes when an input is perturbed, then this input is an adversarial example. Formally, an input $\bm{x'}$ is an adversarial example of an input-label pair $(\bm{x}, y)$ if $\big(h(\bm{x})= y\big)\land \big(h(\bm{x'})\neq h(\bm{x})\big)\land \big(\bm{x'}\in\mathbb{V}({\bm{x}})\big)$~\cite{goodfellow2014explaining,kurakin2016adversarial}, where $\mathbb{V}(\bm{x})$ is the vicinity at $\bm{x}$. We define robustness to be the probability of \emph{not} observing an adversarial example~\cite{lin2019robustness}, as defined in \cref{eq:robustness}.
\begin{equation}
\label{eq:robustness}
    \Upsilon^+_\text{rob}(D, h, \mathbb{V}) = P_{(\mathbf{x}, \mathbf{y})\sim\,D} \Big(\big(h(\mathbf{x}) = \textnormal{y}) \big) \land \forall\bm{x'}\in\mathbb{V}(\mathbf{x}).~ h(\bm{x'}) =  h(\mathbf{x}) \Big)
\end{equation}

\begin{remark}[Vicinity]
    $\bm{x}$-vicinity also has an equivalent distribution notation $\mathcal{V}(\bm{x})$. Let $(v: \mathbb{R}^n\to\mathbb{R})$ denote its probability density function (PDF) such that $v(\bm{x'}-\bm{x})>0$ means $\bm{x'}$ is in $\bm{x}$-vicinity.
    $v$ is an even and quasiconcave function \emph{e.g.}, the Gaussian function. \cref{sec:vicinity} shows the equivalence.

\end{remark}


\vspace{-0.05mm}
\textbf{Deterministic Robustness}~
Deterministic robustness requires a zero probability of adversarial examples occurring in a vicinity. It is difficult as achieving $\forall\bm{x'}\in\mathbb{V}(\mathbf{x}), \, h(\bm{x'})=h(\mathbf{x})$ is challenging. Although adversarial training~\cite{goodfellow2014explaining} empirically reduces adversarial examples~\cite{ganin2016domain}, it lacks a formal guarantee. Meanwhile, certified training~\cite{muller2022certified} guarantees deterministic robustness through integrating NN verification during training, often resulting in a significant accuracy drop (35\%) ~\cite{li2023sok}.

\vspace{-0.05mm}

\textbf{Probabilistic Robustness}~ While deterministic robustness is often infeasible without seriously compromising accuracy, probabilistic robustness claims to balance robustness and accuracy~\cite{robey2022probabilistically}. Probabilistic robustness is defined as in \cref{eq:probrob}, where a tolerance level $\kappa$ limits the probability of having adversarial examples in a vicinity $\mathcal{V}$. Here, a small portion (such as 1\%~\cite{zhang2023proa}-50\%~\cite{cohen2019certified}) of adversarial examples within the vicinity is considered acceptable. Probabilistic robustness is often sufficient in 
practice~\cite{robey2022probabilistically}. Indeed, safety certification of many safety-critical domains such as aviation requires 
keeping safety violation probabilities below a non-zero threshold~\cite{guerin2021certifying}.

\begin{equation}
\label{eq:probrob}
    \Upsilon^+_\text{prob}(D, h, \mathcal{V}, \kappa) = P_{(\mathbf{x}, \mathbf{y})\sim\,D} \Big(\big(h(\mathbf{x}) = \textnormal{y}\big) \land \left(\big(P_{\mathbf{x'}\sim\,\mathcal{V}(\mathbf{x})} (h(\mathbf{x'}) \neq h(\mathbf{x})\big) \le \kappa\right) \Big)
\end{equation}


\subsection{An Upper Bound of Deterministic Robustness from Bayes Error}

\paragraph{The Bayes Error}
In the presence of uncertainty in data distribution, a classifier (no matter how it is trained) inevitably makes some wrong predictions. Bayes error quantifies this inherent uncertainty and represents the irreducible error in accuracy~\cite{fukunaga1990introduction,garber1988bounds,ripley1996pattern}, formally captured in \cref{eq:bayeserror1}.
\begin{equation}
\label{eq:bayeserror1}
\begin{aligned}
    \min_{h\in\set{\mathbb{R}^n\to\set{0,1,{\scriptstyle\ldots},K-1}} }\Upsilon^-_\text{acc}(D, h)=\operatorname{E}_{(\mathbf{x}, \textnormal{y}) \sim D}\left[1 - \max_k p(\textnormal{y}=k|\mathbf{x})\right]
\end{aligned}
\end{equation}
A classifier achieves the Bayes error when its predictions correspond to the class with maximal posterior probability. Such a classifier is known as a Bayes classifier. The posterior of other classes thus contributes to the irreducible error. An example illustrating the Bayes error is shown in \cref{fig:runex1}.

\paragraph{Bayes Error for Deterministic Robustness}

Prior work~\cite{zhang2024certified} shows that optimising towards deterministic robustness makes the Bayes error worse. Besides the posterior of other classes, forcing a prediction to be consistent with its neighbours constitutes another source of Bayes error~\cite{zhang2024certified}. As in \cref{eq:ub_compute_2}, the Bayes error for deterministic robustness can be derived from the Bayes error of a convolved distribution $D'=D*v$. In $D'$, $p(\mathbf{x},\mathbf{y})$ is convolved from vicinity $v$ and $p(\mathbf{x},\mathbf{y})$ in $D$.
\begin{equation}
\label{eq:ub_compute_2}
    \min_{h\in\set{\mathbb{R}^n\to\set{0,1,{\scriptstyle\ldots},K-1}} } \Upsilon^-_\text{rob}(D, h, \mathbb{V}) =\operatorname{E}_{(\mathbf{x}, \textnormal{y}) \sim D'}\left[1 - \max_k p(\textnormal{y}=k|\mathbf{x})\mathbf{1}_{\mathbf{x}\notin\mathbb{K}_{D^\dagger}}\right]
\end{equation}
where $\Upsilon^-_\text{rob}(D, h, \mathbb{V})=1-\Upsilon^+_\text{rob}(D, h, \mathbb{V})$. $D^\dagger = \lceil D'\rceil*v$ where $\lceil D'\rceil$ is the ``hardened'' distribution of $D'$, \emph{i.e.}, one-hot of Argmax posterior. $\mathbb{K}_{D^\dagger} = \set{\bm{x}|(\mathbf{x},\textnormal{y}){\scriptstyle\sim} D^\dag, \max_k p(k|\mathbf{x}=\bm{x})<1}$ represents a domain near the boundary where the marginal probability rather than joint probability contributes to the Bayes error for deterministic robustness. Therefore, the Bayes error for deterministic robustness of $D$ is the Bayes error of $D'$ plus the joint probability of non-max classes in $\mathbb{K}_{D^\dagger}$. As shown in~\cite{zhang2024certified}, deterministic robustness $\Upsilon^+_\text{rob}(D, h, \mathbb{V})$ has an upper bound of 1 minus this irreducible error. \cref{fig:runex1_2} illustrates the Bayes error for deterministic robustness.

\subsection{Problem Definition}

The fundamental problem in this study is finding an upper bound of probabilistic robust accuracy. Further, we aim to establish a relation between this upper bound and the tolerance level, \emph{i.e.}, $\kappa$. Formally, we solve the minimisation problem in \cref{eq:problem} where $\Upsilon^-_\text{prob}(D, h, \mathcal{V}, \kappa) = 1 - \Upsilon^+_\text{prob}(D, h, \mathcal{V}, \kappa)$.
\begin{equation}
\label{eq:problem}
    \min_{h\in\set{\mathbb{R}^n\to\set{0,1,{\scriptstyle\ldots},K-1}} } \Upsilon^-_\text{prob}(D, h, \mathcal{V}, \kappa)
\end{equation}

\begin{figure}[t]
    \centering
    \begin{subfigure}[t]{0.26\linewidth}
        \includegraphics[width=\linewidth]{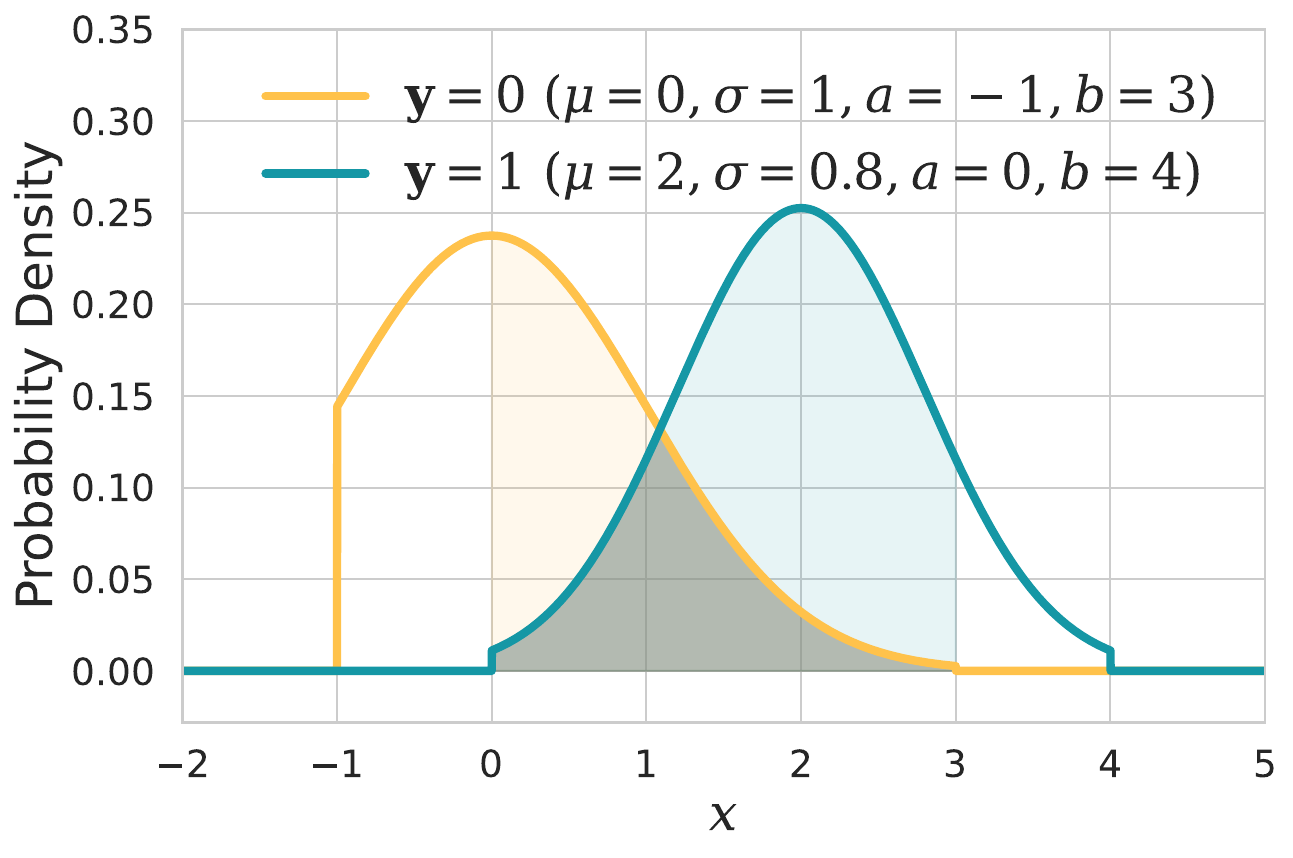}
        \caption{Bayes error
        }
        \label{fig:runex1}
    \end{subfigure}
    \hfill
    \begin{subfigure}[t]{0.26\linewidth}
        \includegraphics[width=\linewidth]{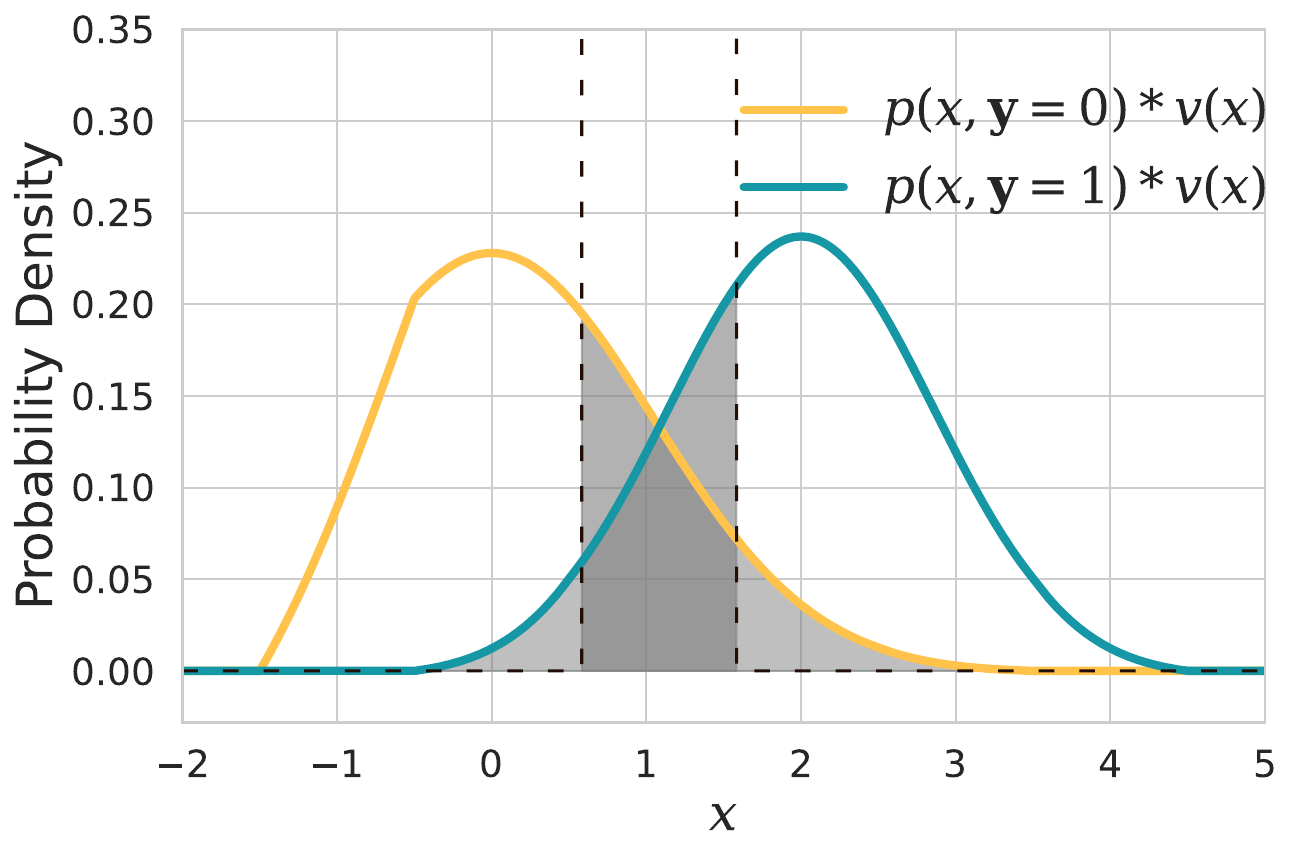}
        \caption{$\epsilon=0.5$
        }
        \label{fig:runex1_2}
    \end{subfigure}
    \hfill
    \begin{subfigure}[t]{0.26\linewidth}
        \includegraphics[width=\linewidth]{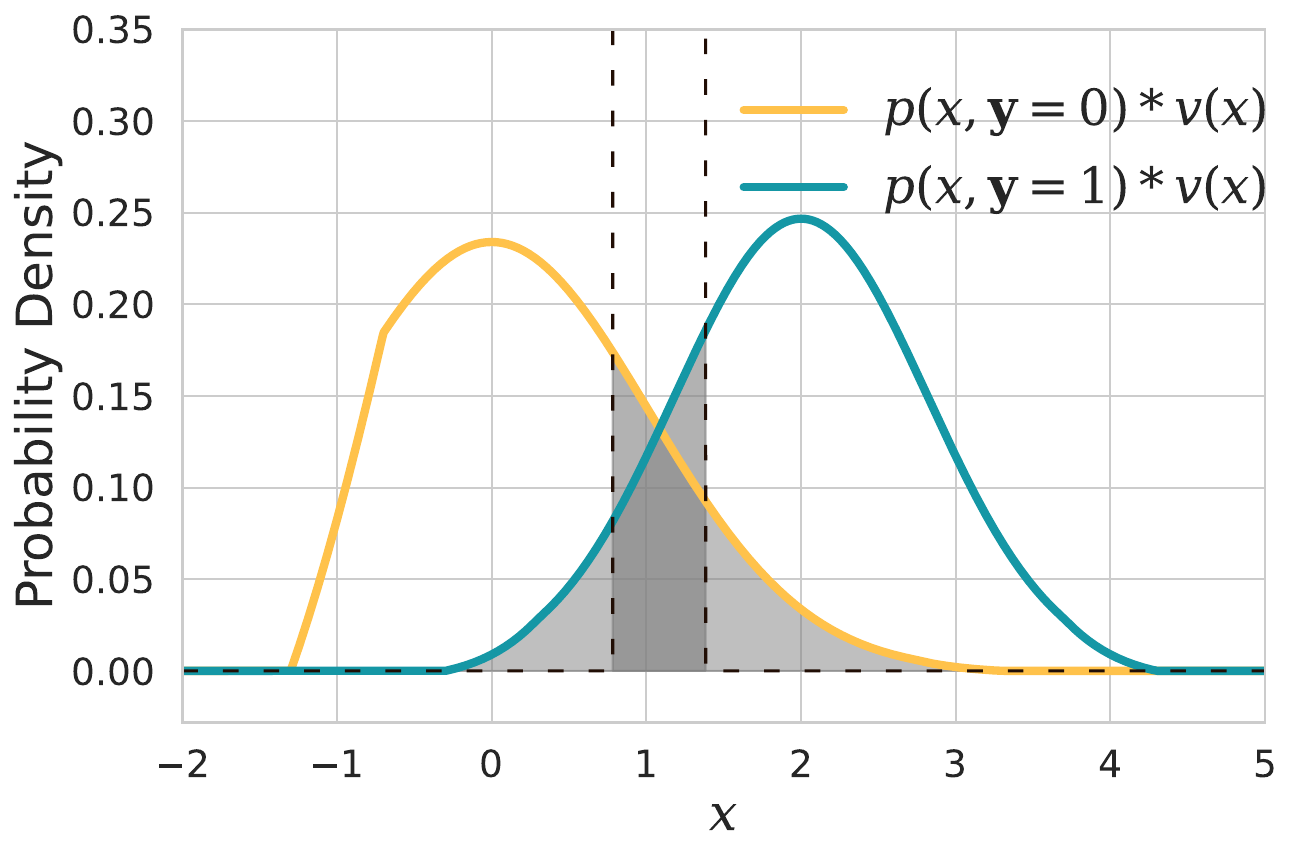}
        \caption{$\epsilon=0.5, \kappa = 0.2$
        }
        \label{fig:runex1_3}
    \end{subfigure}
    \hfill
    \begin{subfigure}[t]{0.190\linewidth}
        \includegraphics[width=\linewidth, trim=0.cm 0.cm 1.1cm .45cm, clip]{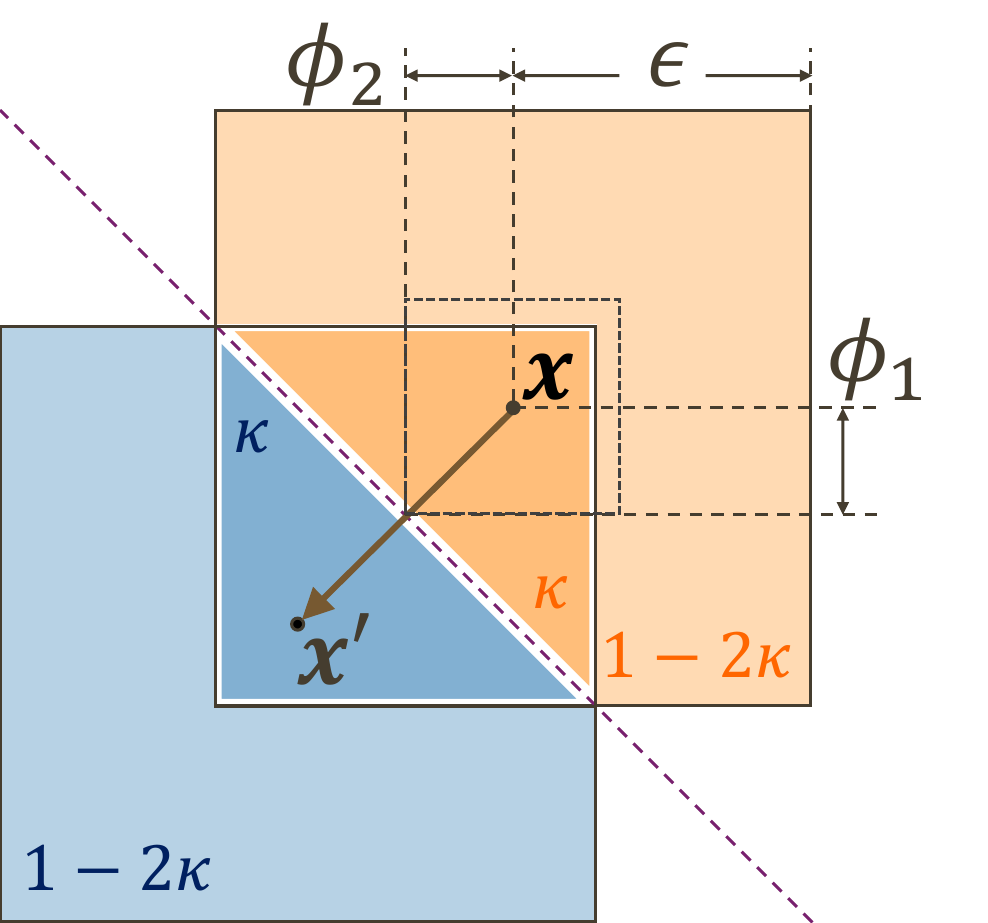}
        \caption{$L^\infty$ 2D example}
        \label{fig:2d_square_example}
    \end{subfigure}
    \caption{ Two truncated normal distributions are used to visualise the Bayes error of (a) vanilla accuracy, (b) deterministic robust accuracy and (c) probabilistic robust accuracy. (d) Example of \cref{cor:linf}. The nearest adversarial example of $\bm{x}$ is at the midpoint of $\bm{x}$ and $\bm{x'}$. Both $\bm{x}$ and $\bm{x'}$ are probabilistically robust but $h(\bm{x})\neq h(\bm{x'})$. The dashed box with side length $2\phi_i$ representes $\mathbb{V}^{\downarrow\kappa}(\bm{x})$.}
    \label{fig:pre}
\end{figure}

\section{Method}
\label{sec:method}

In the following, we study the upper bound of probabilistic robust accuracy. We first model the error when optimising towards probabilistic robustness and derive an optimal decision rule. Then, we study the Bayes error obtained from this rule. Further, we formally establish the relationship between the upper bounds of vanilla accuracy, probabilistic robust accuracy, and deterministic robust accuracy.


\subsection{Error Modelling and Optimal Decision Rule for Probabilistic Robustness}
\label{sec:modelling}

To find Bayes error when optimising towards probabilistic robustness given distribution $(\mathbf{x},\textnormal{y}){\scriptstyle\sim} D$, we first model $\Upsilon^-_{\textsc{p}\text{rob}}(D, h, \mathcal{V}, \kappa)$. Intuitively, an error happens if the prediction is wrong or many of the samples in the vicinity are predicted wrongly. We denote the error from the former case as incorrectness and the latter as inconsistency. We analyze each type of error and their combined effect.


\paragraph{Incorrectness}
The incorrectness for any example $(\bm{x}, y)$ is simply $\mathbf{1}_{h(\bm{x}) \neq y}$. The incorrectness of a prediction at an input $\bm{x}$ given all possible labels $y \in \set{0, 1, {\scriptstyle\ldots}, K-1}$ considers posterior at $\bm{x}$, as in \cref{eq:cor}. Incorrectness is minimum when $h(\bm{x})$ equals the class with the highest posterior.
\begin{equation}
\label{eq:cor}
    e_\text{cor}(\bm{x}, h; P(\textnormal{y}\mid\bm{x}))
    = \sum_{y=0}^{K-1} P(\textnormal{y}=y\mid \mathbf{x}=\bm{x}) \,\mathbf{1}_{h(\bm{x}) \neq y} = 1-\sum_{y=0}^{K-1} P(\textnormal{y}=y\mid \mathbf{x}=\bm{x}) \,\mathbf{1}_{h(\bm{x}) = y}
\end{equation}

\paragraph{Inconsistency}
Inconsistency results from prediction at $\bm{x}$ being not the same as some of its neighbours. Let $P_{\mathbf{t}\sim\,\mathcal{V}(\bm{x})}\big(h(\mathbf{t}) \neq  h(\bm{x})\big)$ denote the probability of a neighbour of $\bm{x}$ having a prediction different from $\bm{x}$. Since this probability is parameterised by $h(\bm{x})\in\set{0, 1, {\scriptstyle\ldots}, K-1}$, it can be reformulated as $\sum_{k=0}^{K-1} \mathbf{1}_{h(\bm{x}) = k} P_{\mathbf{t}\sim\,\mathcal{V}(\bm{x})}\big(h(\mathbf{t}) \neq k\big) $. 
Let $\mu_k(\bm{x}) = P_{\mathbf{t}\sim\,\mathcal{V}(\bm{x})}\big(h(\mathbf{t}) = k\big)$ and $\sum _{k=0}^{K-1}\mu_{k}(\bm{x})=1$. Intuitively, $\mu_k$ is the probability of a neighbour predicted as class-$k$. Thus, the probability of a neighbour
of $\bm{x}$ having a different prediction from $\bm{x}$ can be written as \cref{eq:cns_inner}.
\begin{equation}
\label{eq:cns_inner}
    P_{\mathbf{t}\sim\,\mathcal{V}(\bm{x})}\big(h(\mathbf{t}) \neq  h(\bm{x})\big) = \sum_{k=0}^{K-1} \mathbf{1}_{h(\bm{x}) = k} \big(1-\mu_k(\bm{x})\big) =1-\sum_{k=0}^{K-1} \mu_k(\bm{x})\,\mathbf{1}_{h(\bm{x}) = k}
\end{equation}

Inconsistency exists when $P_{\mathbf{t}\sim\,\mathcal{V}(\bm{x})}\big(h(\mathbf{t}) \neq  h(\bm{x})\big)>\kappa$. This thresholding can be represented by a unit step function ($u$) that takes an input $P_{\mathbf{t}\sim\,\mathcal{V}(\bm{x})}\big(h(\mathbf{t}) \neq h(\bm{x})\big) - \kappa$. Thus, inconsistency at $\bm{x}$ is expressed as \cref{eq:cns_simp}. Also, \cref{le:half} suggests that $\kappa$ takes value from $[0, 1/2)$.
\begin{equation}
\label{eq:cns_simp}
    e_\text{cns}(\bm{x}, h; \mathcal{V}, \kappa) 
    = u\left(P_{\mathbf{t}\sim\,\mathcal{V}(\bm{x})}\Big(h(\mathbf{t}) \neq  h(\bm{x})\Big) -\kappa\right) = u\left(1- \kappa - \sum_{k=0}^{K-1} \mu_k(\bm{x})\,\mathbf{1}_{h(\bm{x}) = k} \right)
\end{equation}

\begin{lemma}
\label{le:half}
For the prediction of input $\bm{x}$ to be consistent, at most one class has a prediction probability $\ge 1-\kappa$ in $\bm{x}$-vicinity. Thus, $\kappa < \frac{1}{2}$.
(Proof is provided in \cref{app:half}.)
\end{lemma}

\paragraph{The overall error across the distribution}

Considering probabilistic robustness, the error at input $\bm{x}$ is a combined error of $e_\text{cor}$ and $e_\text{cns}$ at $\bm{x}$. 
We need two intuitions to derive the combined error. First, if $e_\text{cns}(\bm{x}, h; \mathcal{V}, \kappa)=1$, the combined error is always 1. Second, if $e_\text{cns}(\bm{x}, h; \mathcal{V}, \kappa)=0$, the combined error equals $e_\text{cor}(\bm{x}, h; P(\textnormal{y}|\bm{x}))$. Note that $e_\text{cor}$ takes value from $[0,1]$ and $e_\text{cns}$ takes binary value from $\set{0,1}$. The combined error $e$ is expressed as \cref{eq:combine} whose derivation is in \cref{app:combine}.
\begin{equation}
\label{eq:combine}
\begin{aligned}
    e(\bm{x}, h; P(\textnormal{y}&|\bm{x}), \mathcal{V}, \kappa) = (1 - e_\text{cns}(\bm{x}, h;\mathcal{V},\kappa)) e_\text{cor}(\bm{x}, h; P(\textnormal{y}\mid\bm{x})) + e_\text{cns}(\bm{x}, h;\mathcal{V},\kappa) \\
    &= 1-u\left(\kappa - 1 + \sum_{k=0}^{K-1} \mu_k(\bm{x})\,\mathbf{1}_{h(\bm{x}) = k} \right)\left(\sum_{y=0}^{K-1} P(\textnormal{y}=y\mid \mathbf{x}=\bm{x}) \,\mathbf{1}_{h(\bm{x}) = y} \right)
\end{aligned}
\end{equation}
Note that in general, the errors are functions of $\bm{x}, h, D,\mathcal{V}$, and $\kappa$. For simplicity, when $h$, $D$, $\mathcal{V}$, or $\kappa$ can 
be inferred from the context, we simply omit them, \emph{e.g.}, the simplest case is written as $e_\text{cor}(\bm{x})$ to denote the incorrectness, $e_\text{cns}(\bm{x})$ to denote inconsistency, and $e(\bm{x})$ to denote the combined error.


To model the distribution-wise error of classifier $h$ on $(\mathbf{x}, \textnormal{y}){\scriptstyle\sim}D$, we compute the expectation of $e(\bm{x})$ across $D$. Formally, $\Upsilon^-_\text{prob}(D, h, \mathcal{V}, \kappa) = \int_{\bm{x} \in \mathbb{R}^n} e(\bm{x}) p(\mathbf{x} = \bm{x}) d\bm{x}$, where $p(\mathbf{x})$ is the marginal probability in $D$. Hereby, we get $\Upsilon^-_\text{prob}$, the error when optimising towards probabilistic robustness.

To minimise $\Upsilon^-_\text{prob}$ of any measurable classification function $h$, we explore the optimal decision rules for probabilistic robustness. From \cref{eq:combine}, we can establish a Maximum A Posteriori optimal decision rule, whose formal statement is given in \cref{thm:align}.

\begin{theorem}
\label{thm:align}
    If $h^*$ is optimal for the probabilistic robustness on a given distribution, \emph{i.e.}, $h^* = \argmin_h \int_{\bm{x} \in \mathbb{R}^n} e(\bm{x})p(\mathbf{x} = \bm{x}) d\bm{x}$, we would always have
    $\forall \bm{x} \in \mathbb{R}^n, h^*(\bm{x}) = \argmax_k \mu_k(\bm{x})$.
\end{theorem}

\begin{proof}
    Let $h_1$ and $h_2$ be two distinct classification functions such that $h_1(\bm{x}) = \argmax_k \mu_k(\bm{x})$ and $h_2(\bm{x}) \neq h_1(\bm{x})$. If we can prove $e(\bm{x}, h_1) \le e(\bm{x}, h_2)$, then we can know $h_1$ must be optimal for probabilistic robustness. First, we denote $k_1 = h_1(\bm{x})$ and $ k_2 = h_2(\bm{x})\neq k_1$. Then,
    \begin{equation}
        e(\bm{x}, h_1) - e(\bm{x}, h_2)
        =u\big(\kappa - 1 + \mu_{k_2}(\bm{x}) \big)P(\textnormal{y}=k_2| \mathbf{x}=\bm{x})  - u\left(\kappa - 1 + \mu_{k_1}(\bm{x}) \right)P(\textnormal{y}=k_1| \mathbf{x}=\bm{x}).
    \end{equation}
    Since $\mu_{k_2}(\bm{x}) \le \mu_{k_1}(\bm{x})$, we get $\mu_{k_2}(\bm{x}) \le 1/2$. Recall $\kappa < 1/2$ from \cref{le:half}, we get $\kappa - 1 + \mu_{k_2}(\bm{x}) < 1/2 - 1 + 1/2 = 0$. Consequently, we have $u\left(\kappa - 1 + \mu_{k_2}(\bm{x}) \right) = 0$. Therefore, $e(\bm{x}, h_1) - e(\bm{x}, h_2) =  - u\left(\kappa - 1 + \mu_{k_1}(\bm{x}) \right)P(\textnormal{y}=k_1| \mathbf{x}=\bm{x}) \le 0$. This inequality applies to any input $\bm{x}$. Hence, a classification function like $h_1$ is optimal.
    An extended proof is in \cref{app:align}.
\end{proof}
Intuitively, the theorem states that when optimising towards probabilistic robustness, a Bayes (optimal) classifier would always classify a sample with the most popular label in the vicinity.

\subsection{Bayes Error for Probabilistic Robustness from Bayes Error for Deterministic Robustness}
\label{sec:bound}

The Bayes classifier, regarding probabilistic robustness, is closely related to the most popular label in its vicinity, leading us to study the properties of $\mu_k$.
Intuitively, $\mu_k$ is the probability of a neighbour predicted as class-$k$. Formally, $\mu_k$ has an equivalent convolutional form as
\begin{equation}
\label{eq:mu_conv}
    \mu_k(\bm{x}) = P_{\mathbf{t}\sim\,\mathcal{V}(\bm{x})}\big(h(\mathbf{t}) = k\big) =\int_{\bm{t}\in\mathbb{R}^n}\mathbf{1}_{h(\bm{t})=k} ~v(\bm{x} - \bm{t})\,d\bm{t} = (\mathbf{1}_{h(\cdot)=k} * v)(\bm{x}),
\end{equation}
where $*$ denotes convolution and $\mathbf{1}_{h(\cdot)=k}$ denotes an indicator function returning 1 if $h$ of input equals $k$. $v$ is the probability density function of vicinity distribution, \emph{e.g.}, uniform distribution.

Intuitively, convolution acts as a smoothing operation. Thus, $\mu_k(\bm{x})$ is expected to change \emph{gradually} as $\bm{x}$ moves in $\mathbb{R}^n$. Similarly, $\argmax_k \mu_k(\bm{x})$ is unlikely to switch frequently. This implies that under the optimal probabilistic robustness condition, predictions do not change randomly or frequently  (in $\mathbb{R}^n$) but exhibit a form of continuity. \cref{lemma:bound_change} and \cref{thm:continuous} formally states this intuition. Specifically, \cref{lemma:bound_change} states that $\mu_k(\bm{x})$ changes \emph{gradually} as $\bm{x}$ moves in $\mathbb{R}^n$. Moreover,  \cref{thm:continuous} states that
the Bayes classifier (for probabilistic robustness) achieves deterministic robustness with a much smaller vicinity at any input that achieves probabilistic robustness.

\begin{lemma}
\label{lemma:bound_change}
    The change in $\mu_k$ resulting from shifting an input by a certain distance $\phi$ within the vicinity is bounded in any direction $\bm{\hat{\phi}}$. Formally: where $\mathbb{S}^{n-1}$ is the set of all unit vectors in $\mathbb{R}^n$,
    \begin{equation}
    \label[ineq]{ineq:mu_change}
    \begin{aligned}
        \forall\bm{x} \in \mathbb{R}^n,\forall \phi\in\mathbb{R},        \Bigg(\left(\forall \bm{\hat{\phi}}\in\mathbb{S}^{n-1},~ v\left(\frac{\phi}{2}\bm{\hat{\phi}}\right)>0\right) \to\\
        \forall \bm{\hat{\phi}}\in\mathbb{S}^{n-1},~ \left( \abs{\mu_k(\bm{x} + \phi\bm{\hat{\phi}}) - \mu_k(\bm{x})} \le  1 - \min_{\bm{\hat{\phi'}}\in\mathbb{S}^{n-1}}\int_{\bm{t}\in\mathbb{R}^n}\min \left( v(\bm{t}  -\phi\bm{\hat{\phi'}}),  v(\bm{t})\right) d\bm{t}\right)\Bigg).
    \end{aligned}
    \end{equation}
\end{lemma}

Proof of \cref{lemma:bound_change} is given in \cref{app:bound_change}. Essentially, for all inputs shifting a distance $\phi$, the $\mu_k$ value difference between the original input and the shifted input will be bounded by the 1 minus the minimum overlap between two vicinities that are $\phi$ apart. From \cref{ineq:mu_change}, the correlation between the distance $\phi$ and maximum change of $\mu_k$ can be modelled as a function of $\phi$ expressed as
\begin{equation}
\label{eq:bijective}
    \mathop{\Delta}\limits_{\max}\mu_k (\phi)
    = \max_{\bm{\hat{\phi}}\,\in\,\mathbb{S}^{n-1}, \bm{x}\in\mathbb{R}^n}  \abs{\mu_k(\bm{x}) - \mu_k(\bm{x} + \phi\bm{\hat{\phi}}) }
\end{equation}

Note that $\Delta_{\max}\mu_k(\phi)$ is a monotonic function, \emph{i.e.}, a greater shift distance $\phi$ is required if the maximal change in $\mu_k$ needs to be increased. \cref{thm:continuous} leverages this monotonicity to formally show the connection between the Bayes error when optimising towards probabilistic robustness and that towards deterministic robustness.

\begin{theorem}
\label{thm:continuous}
    If $h^*$ is optimal for the probabilistic robustness on a given distribution, \emph{i.e.}, $h^* = \argmin_{h\in \set{\mathbb{R}^n\to \set{0, 1, {\scriptstyle\ldots}, K-1}} }\Upsilon^-_\textnormal{prob}(D, h, \mathcal{V}, \kappa)$, then there is a lower bound on the distance between an input $\bm{x}$ and any of its adversarial examples if probabilistic consistency is satisfied on $\bm{x}$. Formally,
    \begin{equation}
    \begin{aligned}
        \forall \bm{x}\in \mathbb{R}^n.~ \Big(\big(\exists k \in \set{0, 1, {\scriptstyle\ldots}, K-1}.~ \mu_k(\bm{x}) > 1 - \kappa\big) \to\\
        \forall \bm{x'}\in \mathbb{R}^n.~ \left((h^*(\bm{x'}) = k) \lor \left(\abs{\bm{x}- \bm{x'}} \geq (\Delta_{\max}\mu_k)^{-1}(1/2-\kappa)\right)\right)\Big).
    \end{aligned}
    \end{equation}
\end{theorem}

\begin{proof}
Suppose an input $\bm{x}$ whose prediction $h(\bm{x})$ is consistent, \emph{i.e.}, $\exists k.\, \mu_k(\bm{x})\ge 1-\kappa$. Let $k^*$ denote this predicted class. Let $\phi_1$ and $\phi_2$ denote two scalar distances to shift $\bm{x}$ and assume the following features of these two distances. First, $\Delta_{\max}\mu_{k^*}(\phi_1) = 1/2-\kappa$ and $\Delta\mu_{k^*}(\phi_2, \bm{\hat{\phi}}_2) > 1/2-\kappa$. The latter indicates that in some direction, moving the input by a distance of $\phi_2$ results in a change in $\mu_{k^*}$ greater than $1/2-\kappa$. Second, $\phi_1 > \phi_2$. From the first condition, we can derive that
\begin{equation}
\label[ineq]{ineq:mono1}
    \min_{\bm{\hat{\phi}}_1\,\in\,\mathbb{S}^{n-1}}\int_{\bm{t}\in\mathbb{R}^n}\min \left( v(\bm{t}  -\phi_1\bm{\hat{\phi}}_1),  v(\bm{t})\right) d\bm{t}
>\int_{\bm{t}\in\mathbb{R}^n}\min \left( v(\bm{t}  -\phi_2\bm{\hat{\phi}}_2),  v(\bm{t})\right) d\bm{t}.
\end{equation}
In other words, we can find some vector $\phi_2\bm{\hat{\phi}}_2$ such that this shift results in a vicinity overlap smaller than the minimum vicinity overlap caused by a $\phi_1$-magnitude shift.

Next, we further shift $\bm{x}+\phi_2\bm{\hat{\phi}}_2$ along $\bm{\hat{\phi}}_2$ direction but with the magnitude $\phi_1  - \phi_2$. This new position, $\bm{x}+\bm{\hat{\phi}}_2\phi_1$, is farther away from $\bm{x}$ than $\bm{x}+\phi_2\bm{\hat{\phi}}_2$ is because $\phi_1 > \phi_2$. Additionally, $\bm{x}+\bm{\hat{\phi}}_2\phi_1$ results in at most the same vicinity overlap (size) as $\bm{x}+\phi_2\bm{\hat{\phi}}_2$ does because $v$ is quasiconcave. Formally, 
\begin{equation}
\label[ineq]{ineq:mono2}
    \int_{\bm{t}\in\mathbb{R}^n}\min \left( v(\bm{t}  -\phi_2\bm{\hat{\phi}}_2),  v(\bm{t})\right) d\bm{t}
    \ge \int_{\bm{t}\in\mathbb{R}^n}\min \left( v(\bm{t}  -\phi_1\bm{\hat{\phi}}_2),  v(\bm{t})\right) d\bm{t}.
\end{equation}
Observe that \cref{ineq:mono2} contradicts (\ref{ineq:mono1}). Therefore, the two assumptions cannot hold simultaneously. An adversarial example requires $\bm{x'}\in\mathbb{V}(\bm{x}), \mu_{k^*}(\bm{x'}) \le 1/2$. Thus,
if $\Delta_{\max}\mu_{k^*}(\phi_1) = 1/2-\kappa$, the distance between a consistent input and any of its adversarial examples is greater than (or equal to) $\phi_1$. The rationale of $1/2-\kappa$ is proven in \cref{app:continuous}.
\end{proof}
Intuitively, at optimal probabilistic robust accuracy, if an input has a probabilistically consistent prediction, all its neighbours in some specific vicinity are predicted the same. Namely, the prediction at this input is deterministically robust within this (likely smaller) vicinity.
\cref{cor:slope} suggests that the bounds stated in \cref{thm:continuous} persist even as the shift approaches zero.
\begin{corollary}
\label{cor:slope}
    There exists a finite real value such that for all inputs, the directional derivative value with respect to any arbitrary nonzero vector $\bm{\hat{\phi}}$ (unit vector) does not exceed this finite value and does not fall below the negative of this value. Formally: where $\cdot$ denotes the dot product,
    \begin{equation}
        \exists b\in\mathbb{R}, \forall \bm{x}\in \mathbb{R}^n, \forall\bm{\hat{\phi}}\in\mathbb{S}^{n-1},~  -b \leq \nabla \mu_k(\bm{x} )\cdot \bm{\hat{\phi}} \leq b.
    \end{equation}
\end{corollary}

Proof of \cref{cor:slope} is provided in \cref{app:slope}. A one-dimensional example in \cref{app:uniform1d} illustrates this idea. We then present an implication of \cref{thm:continuous} for $L^\infty$-norm in \cref{cor:linf}.

\begin{corollary}
\label{cor:linf}
    If $h^*$ is optimal for the probabilistic robustness with respect to an $L^\infty$-vicinity on distribution $D$, then the vicninty size of the deterministically robust vicinity $\mathbb{V}^{\downarrow\kappa}$ around each probabilistically consistent input is $\epsilon(1 - (2\kappa)^{\frac{1}{n}})$. 
\end{corollary}

Proof of \cref{cor:linf} is in \cref{app:linf}. We visualise its effect using a two-dimensional example in \cref{fig:2d_square_example}. Let $\bm{x}, \bm{x'}$ be two probabilistically consistent inputs and $h(\bm{x})\neq h(\bm{x'})$. To minimise $\abs{\bm{x'}-\bm{x}}$, $\bm{x'}$ must be in the diagonal direction (as in \cref{cor:linf}). The shift from $\bm{x}$ to $\bm{x'}$ is $-2(\phi_1\bm{\hat{x_1}} + \phi_2\bm{\hat{x_2}})$. Each triangle accounts for $\kappa$ of the original vicinity volume. Box $\mathbb{V}^{\downarrow\kappa}$ has side length $2\phi_i$. Thus, solving $(2\epsilon - 2\phi_i)^2 = 2\kappa(2\epsilon)^2$, we get $\mathbb{V}^{\downarrow\kappa}$ has vicinity size $\phi_1 = \phi_2 = (1 - \sqrt{2\kappa})\epsilon$. Although \cref{cor:linf} concerns $L^{\infty}$, we can analyse other types similarly, \emph{i.e.}, find the direction with the fastest vicinity overlap decrease and measure the distance to the nearest adversarial example.

In brief, an optimal probabilistic robust accuracy has the following implications. \cref{thm:continuous} suggests that each probabilistically robust input is also deterministically robust, but within a much smaller vicinity. \cref{cor:linf} further quantifies the size of this vicinity. Particularly, the order $1/n$ lets vicinity shrink fast as $n$ grows, making probabilistic robust accuracy higher. Overall, deterministic robust accuracy with the smaller vicinity bounds probabilistic robust accuracy. This is formally captured in \cref{thm:bound}. The effect of this reduced Bayes uncertainty is illustrated in \cref{fig:runex1_3}.

\begin{theorem}
\label{thm:bound}
    Given distribution $D$, vicinity with size $\epsilon$, and tolerance level $\kappa$, the probabilistic robust accuracy has an upper bound as shown in \cref{eq:ub}, where $\mathbb{V}^{\downarrow\kappa}$ or $v^{\downarrow\kappa}$ is the ``smaller'' vicinity 
    and $\mathbb{K}$ follows the definition in \cref{eq:ub_compute_2}, denoting the domain near the boundary.
    \begin{equation}
    \label{eq:ub}
    \begin{aligned}
        \Upsilon^+_\textnormal{prob}\left(D, h, \mathcal{V}, \kappa\right) &\le\max_{ h\in\{\mathbb{R}^n\to\set{0,1,{\scriptstyle\ldots},K-1}\} } \Upsilon^+_\textnormal{rob}\left(D, h, \mathbb{V}^{\downarrow\kappa}\right)\\
        &=\operatorname{E}_{(\mathbf{x}, \textnormal{y}) \sim (D*\,v^{\downarrow\kappa})}\left[\max_k p(\textnormal{y}=k|\mathbf{x})\mathbf{1}_{\mathbf{x}\notin\mathbb{K}_{\lceil D\rceil*\,v^{\downarrow\kappa}}}\right]
    \end{aligned}
    \end{equation}
\end{theorem}

\begin{proof}
    \cref{thm:continuous} infers $\forall \bm{x}, \left((P_{\mathbf{x'}\sim\,\mathcal{V}(\bm{x})} (h(\mathbf{x'}) \neq h(\bm{x})) \le \kappa) \to \forall\bm{x'}\in\mathbb{V}^{\downarrow\kappa}(\bm{x}).~ h(\bm{x'}) =  h(\bm{x})\right)$ at optimal probabilistic robust accuracy. Consider this expression in the form of $\forall \bm{x}, \textnormal{Event}_1(\bm{x}) \to \textnormal{Event}_2(\bm{x})$.  This implies that the occurrence rate of Event 2 (deterministic robust accuracy with $\mathbb{V}^{\downarrow\kappa}$) always upper bounds the rate of Event 1 (probabilistic robust accuracy with $\mathbb{V}$). Additionally, the deterministic robust accuracy with $\mathbb{V}^{\downarrow\kappa}$ has an upper bound, as what has been shown in \cite{zhang2024certified}.
\end{proof}

Besides the upper bound, we also find a relatively loose lower bound of probabilistic robust accuracy as shown in \cref{thm:triplet}. This lower bound is the deterministic robust accuracy when the vicinity size is the same as the vicinity size assumed for probabilistic robustness.
\begin{theorem}
\label{thm:triplet}
    The upper bound of probabilistic robust accuracy monotonically increases as $\kappa$ grows. Further, for all tolerance levels $\kappa$, the upper bound of probabilistic robust accuracy lies between the upper bound of deterministic robust accuracy and the upper bound of vanilla accuracy. Formally,
    \begin{equation}
    \begin{aligned}
        \forall\kappa_1,\kappa_2.&\quad (\kappa_1 <\kappa_2)\to \min _{h}\Upsilon^+_{\textnormal{prob}}(D, h, \mathcal{V}, \kappa_1) \le \min _{h}\Upsilon^+_{\textnormal{prob}}(D, h, \mathcal{V}, \kappa_2) \\
        \forall\kappa.& \quad \min _{h} \Upsilon^+_{\textnormal{rob}}(D, h, \mathbb{V}) \le \min _{h}\Upsilon^+_{\textnormal{prob}}(D, h, \mathcal{V}, \kappa) \le \min _{h} \Upsilon^+_{\textnormal{acc}}(D, h)
    \end{aligned}
    \end{equation}    
\end{theorem}
\begin{proof}
    Let $\kappa_1<\kappa_2$ and $k^* =\argmax_k \mu_k(\bm{x})$, and then we get the sign of $e(\bm{x}, \kappa_1) - e(\bm{x}, \kappa_2)$ is the same as $u\big(\kappa_2 - 1 + \mu_{k^*}(\bm{x}) \big)  - u\left(\kappa_1 - 1 + \mu_{k^*}(\bm{x}) \right)$. This is because the posterior probability $P(\textnormal{y}\mid \mathbf{x}=\bm{x})$ is non-negative. Since $\kappa_1<\kappa_2$, and unit step function monotonically increases, we get $e(\bm{x}, \kappa_1) - e(\bm{x}, \kappa_2)\ge 0$. A smaller $\kappa$ leads to a lower or equal upper bound of probabilistic robust accuracy. For deterministic robust accuracy, $\kappa=0$, which is the least value. Thus, deterministic robust accuracy has a smaller upper bound than probabilistic robust accuracy does.
    

    On the other hand, from the intuition of error combination in \cref{sec:modelling}, we get that the combined error is $e(\bm{x}) = 1-(1 - e_\text{cns}(\bm{x}, h;\mathcal{V},\kappa)) (1- e_\text{cor}(\bm{x}, h; P(\textnormal{y}\mid\bm{x})))$. As $0\le e_\text{cns}\le 1$, we get $e(\bm{x}) \ge  1-(1 - e_\text{cor}(\bm{x}, h;\kappa))$. Note that the expectation of $e_\text{cor}(\bm{x}, h;\kappa)$ is the error in vanilla accuracy. Thus, $\Upsilon^+_\text{prob}(D, h, \mathcal{V}, \kappa) \le\Upsilon^+_\text{acc}(D, h)$. So it is with their upper bounds. We also provide an extended proof for this theorem in \cref{app:triplet}.
\end{proof}

In summary, we show that probabilistic robust accuracy is both lower and upper bounded. We also show why the upper bound of probabilistic robust accuracy can be much greater than that of deterministic robust accuracy. Intuitively, our result suggests that probabilistic robustness indeed allows us to sacrifice much less accuracy, compared to that of deterministic robustness. Furthermore, adopting a larger (more relaxed) $\kappa$ (up to 1/2) can effectively increase probabilistic robust accuracy. 


\section{Experiment}
\label{sec:experiment}


We conduct experiments\footnote{Available at \url{https://github.com/soumission-anonyme/irreducible.git}} to validate the above established results empirically. Note that \cref{thm:align} infers voting is optimal. \cref{sec:bound} establishes the probabilistic robust accuracy upper bound. Do these match empirical results? Further, ablation experiments on real-world distribution study how this upper bound changes as $\kappa$ grows. The relationship between accuracy, probabilistic, and deterministic robust accuracy is also studied. In the following, we describe the setups and then answer these questions.

\paragraph{Setup}

Our setup follows prior Bayes error studies~\cite{ishida2022performance,zhang2024certified}. We include four datasets: Moons~\cite{scikit-learn}, Chan~\cite{chen2023evaluating}, FashionMNIST~\cite{xiao2017fashion}, and CIFAR-10~\cite{krizhevsky2009learning}. Given each dataset, we apply a direct method~\cite{ishida2022performance} to compute the Bayes error. $L^\infty$-vicinity is set with $\epsilon = 0.15, 0.15, 0.1, 2/255$ for defining robustness on respective distribution. For deterministic robustness, the Bayes error follows \cref{eq:ub_compute_2}~\cite{zhang2024certified}. For probabilistic robustness, we set $\kappa=0.1$ by default~\cite{robey2022probabilistically} and vary $\kappa$ only for the ablation study. More details are in \cref{app:setup}. Statistic significance is included in \cref{app:sig}.

\paragraph{Does voting always increase probabilistic robust accuracy empirically?}

We first compute the probabilistic robust accuracy of some classifier $h$. We then compute that of a voting classifier $h^\dag$, where $h^\dag(\bm{x}) = \argmax_{k} P_{\mathbf{t}\sim\mathcal{V}(\bm{x})} (k=h(\mathbf{t}))$, with sample size 100. They are compared in \cref{tab:vote}. Training algorithms of $h$ include data augmentation (DA~\cite{shorten2019survey}), randomised smoothing (RS~\cite{cohen2019certified}), and condition value-at-risk (CVaR~\cite{robey2022probabilistically}) which is state-of-the-art (SOTA) for probabilistic robust accuracy. Note that voting always improves probabilistic robust accuracy (at least by +0.1\% or on average +1.58\% ). On DA and RS, the increase is significant (avg + 1.95\%), partly because they are not designed specifically for probabilistic robustness. On CVaR, while modest (avg + 0.85\%), we do observe an increase. This trend is maintained with a larger voting sample size (\cref{app:rq1}).

\begin{table}
  \caption{Probabilistic robustness of classifiers before and after voting. $\kappa=0.1$ and $\epsilon = 0.15, 0.15, 0.1, 2/255$ for Moons, Chan, FashionMNIST, and CIFAR-10.}
  \label{tab:vote}
  \centering
  \resizebox{\textwidth}{!}{
  \begin{tabular}{l*3{>{\centering}p{0.07\linewidth} p{0.2\linewidth}} >{\centering}p{0.07\linewidth} p{0.2\linewidth}}
    \toprule
    &\multicolumn{2}{c}{Moons} &\multicolumn{2}{c}{Chan} &\multicolumn{2}{c}{FashionMNIST} &\multicolumn{2}{c}{CIFAR-10}                   \\
    \cmidrule(r){1-1} \cmidrule(r){2-3} \cmidrule(r){4-5} \cmidrule(r){6-7} \cmidrule(r){8-9}
    DA~\cite{shorten2019survey} & 85.35 & \textbf{85.60}~~(+0.3\%) & 67.96 & \textbf{68.86}~~(+0.9\%) & 84.12 & \textbf{87.48}~~(+3.7\%) & 76.07  & \textbf{81.38}~~(+5.3\%) \\
    \midrule
    RS~\cite{cohen2019certified} &  84.76 & \textbf{85.18}~~(+0.4\%)  & 64.67 & \textbf{66.77}~~(+1.9\%) & 86.29 & \textbf{88.13}~~(+2.1\%) & 87.98  & \textbf{88.95}~~(+1.0\%) \\
    \midrule
    CVaR~\cite{robey2022probabilistically}     & 85.52 & \textbf{85.66}~~(+0.1\%)& 69.46 & \textbf{70.05}~~(+0.6\%) & 88.50 & \textbf{91.07}~~(+2.6\%) & 90.63  & \textbf{90.77}~~(+0.1\%)      \\
    \bottomrule
  \end{tabular}}
\end{table}

\begin{figure}[t]
  \centering
  \hfill
  \begin{subfigure}[t]{\linewidth}
    \centering
    \includegraphics[width=0.8\linewidth, trim=0.5cm 0.4cm 1cm 1.cm, clip]{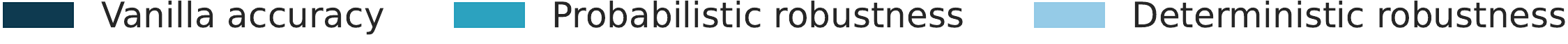}
  \end{subfigure}
  \hfill 
  \begin{subfigure}[t]{0.2\linewidth}
    \centering
    \includegraphics[width=\linewidth]{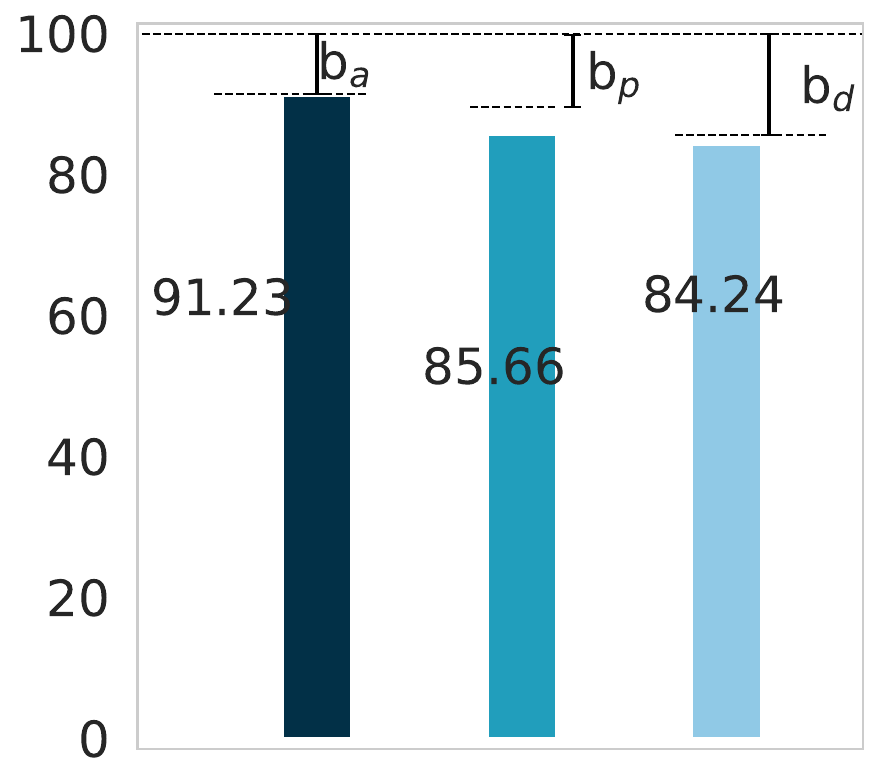}
    \caption{Moons, $b_a= 8.54$, $b_p = 10.13$, $b_d = 14.28$ (\%)}
  \end{subfigure}%
  \hfill
  \begin{subfigure}[t]{0.2\linewidth}
    \centering
    \includegraphics[width=\linewidth]{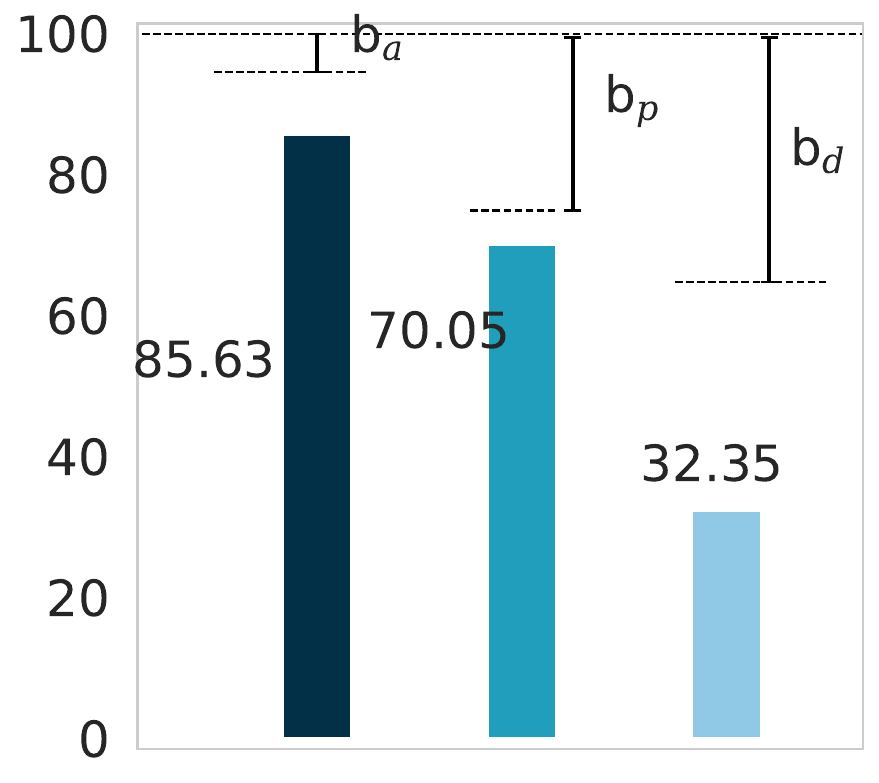}
    \caption{Chan, $b_a= 5.38$, $b_p=24.56$, $b_d= 34.67$ (\%)}
  \end{subfigure}%
  \hfill
  \begin{subfigure}[t]{0.2\linewidth}
    \centering
    \includegraphics[width=\linewidth]{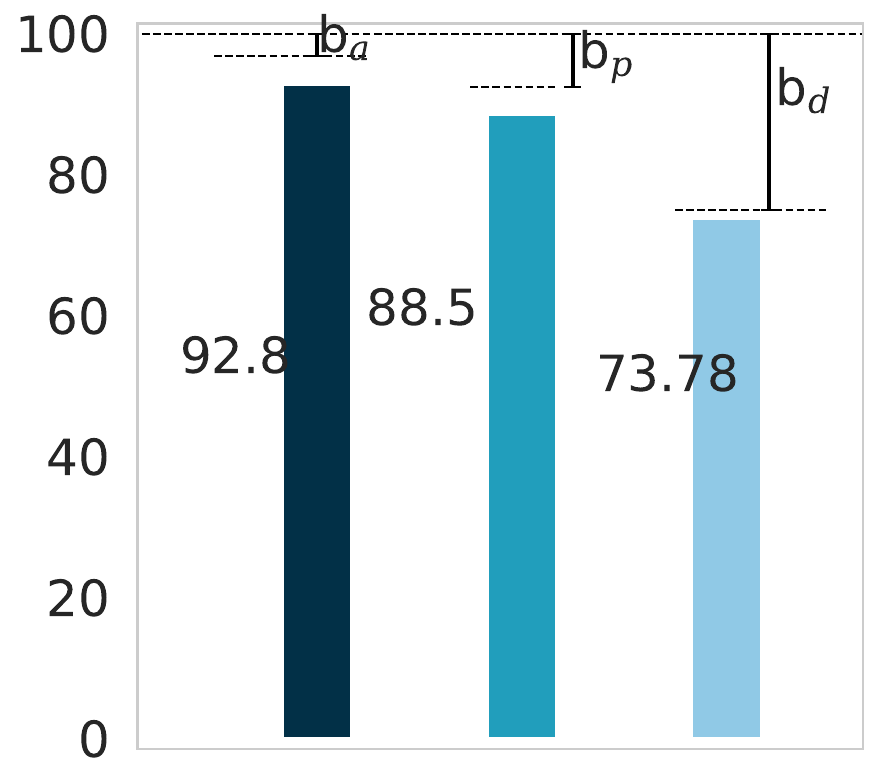}
    \caption{Fashion{\textsc MNIST}, $b_a= 3.15$, $b_p=8.91$, $b_d= 25$ (\%)}
  \end{subfigure}%
  \hfill
  \begin{subfigure}[t]{0.2\linewidth}
    \centering
    \includegraphics[width=\linewidth]{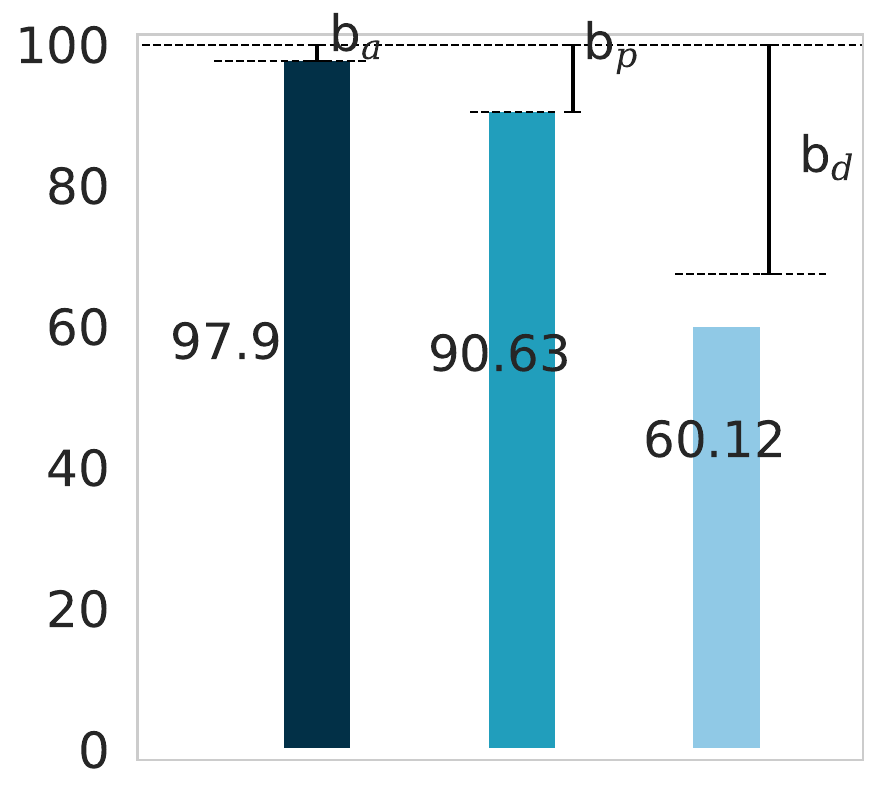}
    \caption{CIFAR-10, $b_a= 5.24$, $b_p=9.21$, $b_d= 32.51$ (\%)}
  \end{subfigure}%
  \hfill
  \caption{Comparing the SOTA classifier performance with upper bounds of vanilla accuracy ($b_a$), probabilisti  c robust accuracy ($b_p$), and deterministic robust accuracy ($b_d$) . 
  \label{fig:accuracy}}
\end{figure}

\paragraph{Is our upper bound empirically valid on existing neural networks?} To check if indeed all trained classifiers respect the theoretical upper bound of probabilistic robust accuracy on any distribution, we compare the SOTA CVaR training and the bound. The middle column of each plot in \cref{fig:accuracy} demonstrates this comparison. We observe that the SOTA probabilistic robust accuracy never exceeds our theoretical bound. Intriguingly, on certain distributions like CIFAR-10, SOTA training almost meets its upper bound with a small gap (0.2\%), while on others, a gap remains (on average 4.06\%). Theoretically, a negative gap may also occur when a classifier overfits the data samples~\cite{ishida2022performance}. Our upper bound is empirically useful in approximating the room for improvement.

\paragraph{How does probabilistic robust accuracy compare to vanilla accuracy and deterministic robust accuracy in terms of upper bounds?}

We observe in \cref{fig:accuracy} that invariably, the upper bound of probabilistic robust accuracy is lower than that of vanilla accuracy and higher than that of deterministic robust accuracy. In high-dimensional distributions, the upper bound of probabilistic robust accuracy is close to that of vanilla accuracy but over 27\% higher than that of deterministic robust accuracy. This could be a result of the curse of dimensionality and much-reduced vicinity size according to \cref{cor:linf}. On Chan, the upper bound of probabilistic robustness is close to that of deterministic robust accuracy but over 20\% lower than that of vanilla accuracy. This could be due to the high-frequency features in the distribution~\cite{zhang2024certified}. On Moons, these three bounds are close (at most 7\% difference). The reason could be that this distribution is relatively smooth.

\paragraph{What is the effect of $\kappa$ on the upper bound of probabilistic robust accuracy?}

Given different $\kappa$, the upper bound of probabilistic robust accuracy can be different. We vary $\kappa$ in $[0, 0.5)$ increasing each time by 0.01. \cref{fig:kappa} shows that this upper bound monotonically grows as $\kappa$ grows, which matches \cref{thm:triplet}. Besides, the growth is fast when $\kappa$ is small (slope $>3$ at $\kappa=0.1$), and the growth rate decreases as $\kappa$ grows (slope $<0.02$ at $\kappa=0.4$). Especially, for high-dimensional distributions, a small change in $\kappa$ when $\kappa$ is small, \emph{e.g.} $<0.1$ can significantly increase the upper bound. This could be explained by the $1/n$ order in \cref{cor:linf}.
This is encouraging as it shows that by sacrificing deterministic robustness only slightly, we can already improve the accuracy significantly.

\begin{figure}[t]
    \centering
    \begin{subfigure}[t]{0.23\linewidth}
    \centering
    \includegraphics[width=\linewidth]{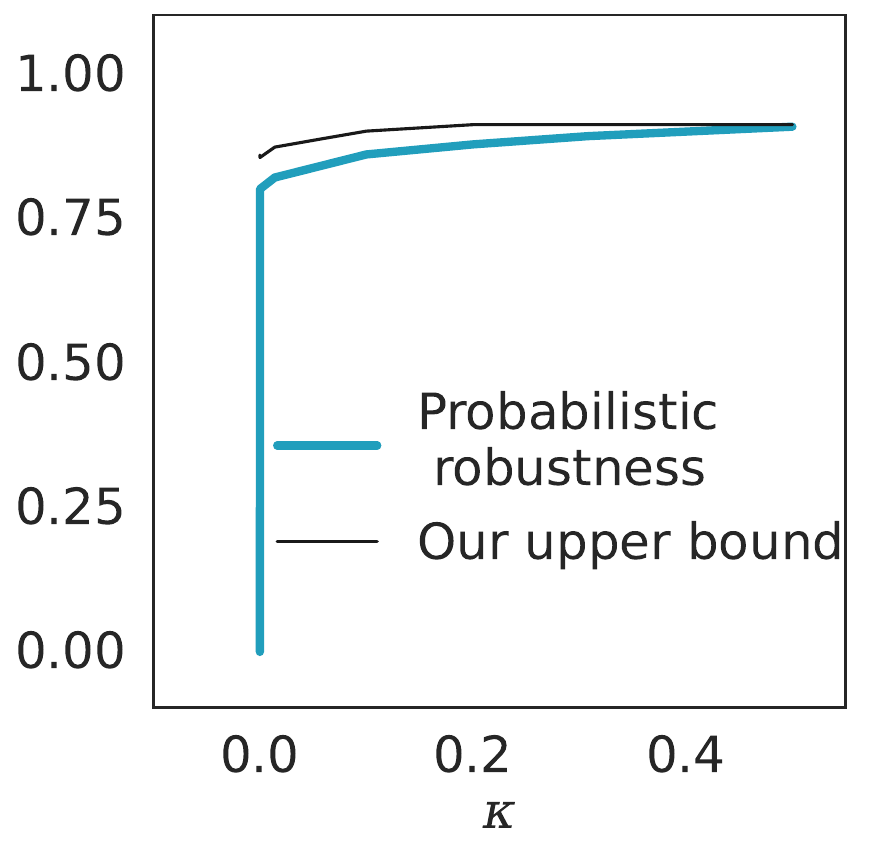}
    \caption{Moons }
    \end{subfigure}%
    \begin{subfigure}[t]{0.23\linewidth}
    \centering
    \includegraphics[width=\linewidth]{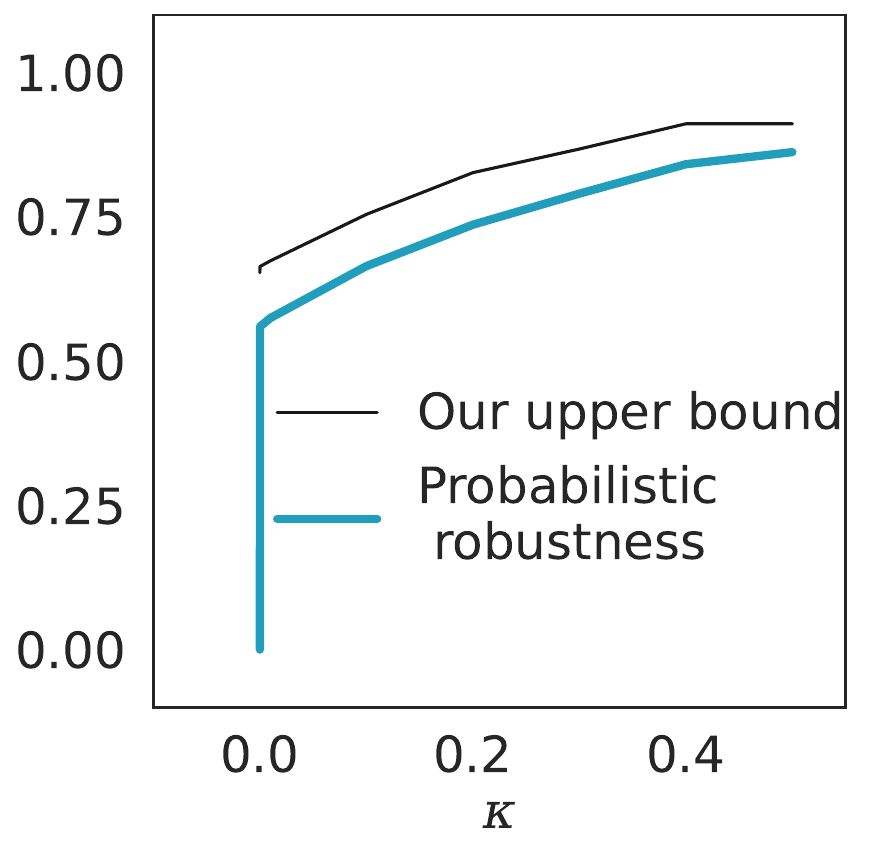}
    \caption{Chan}
    \end{subfigure}%
    \begin{subfigure}[t]{0.23\linewidth}
    \centering
    \includegraphics[width=\linewidth]{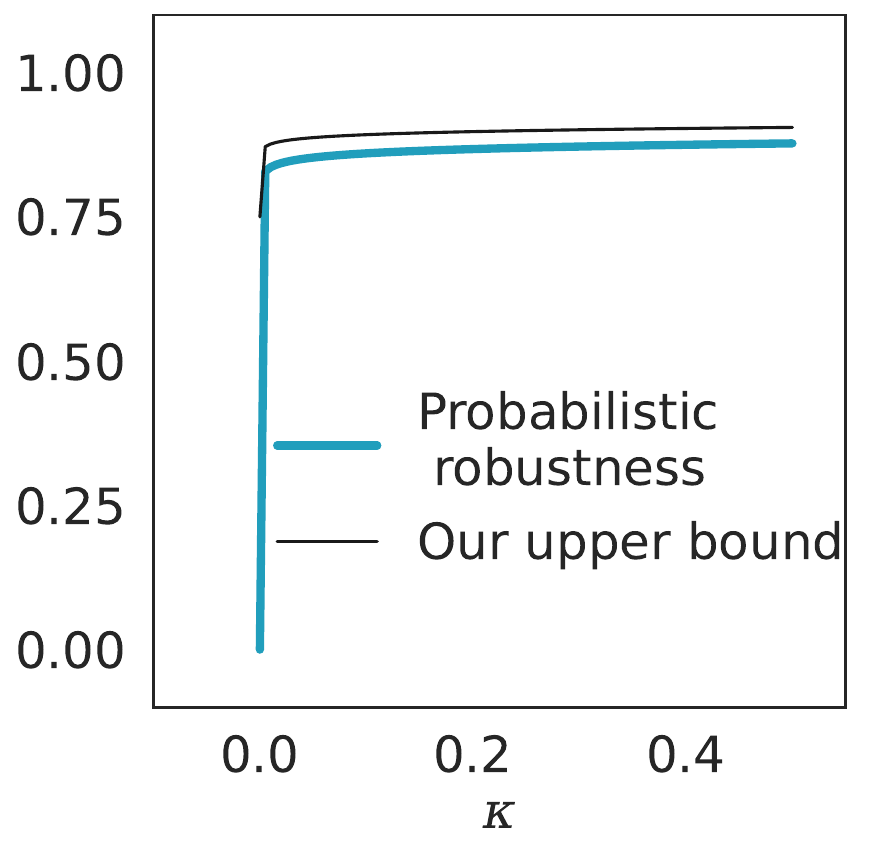}
    \caption{ FashionMNIST}
    \end{subfigure}%
    \begin{subfigure}[t]{0.23\linewidth}
    \centering
    \includegraphics[width=\linewidth]{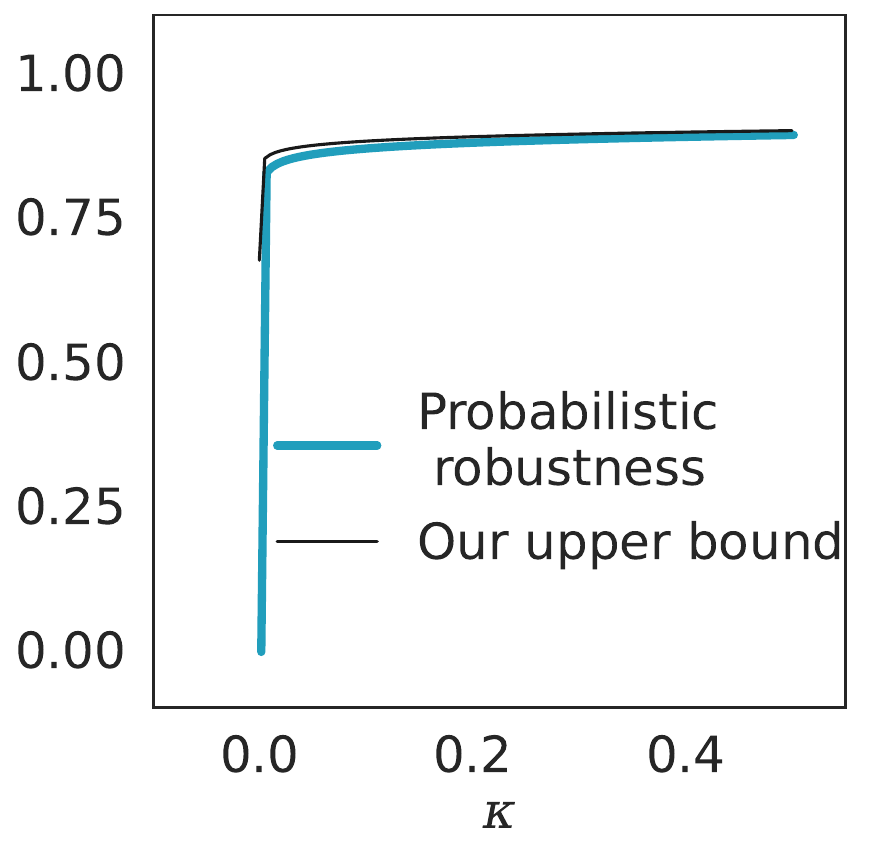}
    \caption{CIFAR-10 }
    \end{subfigure}
    \caption{As $\kappa$ increases, we plot the upper bounds of probabilistic robust accuracy as well as classifiers' probabilistic robust accuracy change in the Moons and Chan dataset.}
    \label{fig:kappa}
\end{figure}

\section{Related Work}
This work is closely related to studies on Bayes errors and probabilistic robustness. Computing the Bayes error of a given distribution has been studied for over half a century~\cite{fukunaga1975k}, and one interesting topic is to derive or empirically estimate the upper and lower bounds of the Bayes error. Various $f$-divergences, such as the Bhattacharyya distance~\cite{fukunaga1990introduction} or the Henze-Penrose divergence~\cite{berisha2015empirically,sekeh2020learning}, have been studied. Other approaches include directly estimating the Bayes error with $f$-divergence representation instead of using a bound~\cite{noshad2019learning}, computing the Bayes error of generative models learned using normalizing flows~\cite{kingma2018glow,theisen2021evaluating}, or evaluate the Bayes error from data sample reassessment~\cite{ishida2022performance}. Recent studies apply Bayes error estimation to deterministic robustness beyond a vanilla accuracy perspective ~\cite{zhang2024certified}. Our study extends this line of research and focuses on probabilistic robustness.

Improving robustness is a core topic in the recent decade~\cite{zhang2023coophance,wang2021robot}. Adversarial training considers adversarial examples in the training phase~\cite{madry2017towards,goodfellow2014explaining,zhang2019theoretically,wang2020improving}. However, adversarially trained models do not come with a theoretical guarantee~\cite{zhang2018efficient,singh2019abstract,balunovic2020adversarial}. Certified training provides this guarantee by optimising bounds from formal verification during training but compromises performance on clean inputs~\cite{shi2021fast,muller2022certified}. To mitigate these problems, probabilistic robustness methods such as PRoA~\cite{zhang2023proa} or CVaR~\cite{robey2022probabilistically} are proposed. Probabilistic robustness offers a desirable balance between robustness and accuracy, making it more applicable in real-world scenarios. 
\section{Conclusion}
\label{sec:conclusion}

We investigate probabilistic robustness by formally deriving its upper bound. We find that the optimal prediction should be the Maximum A Posteriori of predictions in the vicinity. Then, we show that any probabilistically robust input is also deterministically robust within a smaller vicinity. Thus, the upper bound of probabilistic robust accuracy can be obtained from that of deterministic robust accuracy. We verify the empirical impact of this upper bound by comparing it with SOTA training and the upper bounds of vanilla accuracy or deterministic robust accuracy.  Experiments match our theorems and show that our bounds could indicate room for improvement in practice.


\paragraph{Limitation}

A limitation of our upper bound is that \cref{thm:continuous} requires the posterior probability to calculate the probabilistic robust accuracy upper bound whereas it might be difficult to obtain the posterior for some cases. This generally occurs for Bayes uncertainty analyses~\cite{fukunaga1975k} in finding bounds of accuracy~\cite{theisen2021evaluating,nielsen2014generalized,moon2015meta} or deterministic robustness~\cite{zhang2024certified}. Yet, this can be compensated by various density estimation methods~\cite{renggli2021evaluating,zhang2024certified} or reevaluating the probability from a dataset~\cite{ishida2022performance}.

\newpage
\appendix

\section{Notations}

\subsection{Vicinity Notations}
\label{sec:vicinity}
\begin{definition}[Vicinity]
\label{def:vicinity}

    Inputs within the vicinity of an input $\bm{x}$ are imperceptible from $\bm{x}$.  We call this vicinity an $\bm{x}$-vicinity. To capture imperceptibility, an $\bm{x}$-vicinity can be denoted in (at least four) different but equivalent notations, \emph{i.e.}, distance-threshold, set, distribution, and function notations. Occasionally, an input within $\bm{x}$-vicinity is called a neighbour of $\bm{x}$.
\end{definition}

\paragraph{Distance-threshold Notation of Vicinity}

The distance-threshold notation is one of the earliest notations to depict vicinity~\cite{ma2018characterizing, athalye2018synthesizing, bhattacharya2019survey}. Here, the neighbour of a sample $\bm{x} \in \mathbb{X}$ refers to an input that lies within a certain threshold distance from $\bm{x}$. Formally, $\bm{x'}$ is within $\bm{x}$-vicinity if and only if
\begin{equation}
     d(\bm{x'}, \bm{x}) \le \epsilon
\end{equation}
where $d$ measures the distance between two inputs, and this distance needs to be smaller than a threshold $\epsilon$ to be considered imperceptible. 

Specifically, the distance function can be defined in a variety of ways \emph{e.g.}, $L^p$ norm ($p = 0, 1, 2,$ or $ \infty$) or domain-specific transformations that preserve labels, such as tilting or zooming. 
\begin{equation}\label{eq:d}
\begin{aligned}
    d(\bm{x'}, \bm{x}) &= \norm{\bm{x'} - \bm{x}}_p, \quad \text{(Additive in $ L^p $ norm), or}\\
    d(\bm{x'}, \bm{x}) &= 
    \begin{cases}
    \abs{\epsilon}, &\text{if} \quad f_\text{transform} (\bm{x}, \epsilon) = \bm{x'} ,\\
    \epsilon + 1, &\text{otherwise}
    \end{cases}  
\end{aligned}
\end{equation}
where the transformation function $f_\text{transform}: \mathbb{X} \to \mathbb{X}$ can be, for example, an image rotation with a parameter determining the degree of rotation.



\paragraph{Set Notation of Vicinity}

Here, the vicinity of a sample $\bm{x} \in \mathbb{X}$ refers to a set containing all neighbours of $\bm{x}$. Given an input, all inputs whose distance to the given input is within a certain threshold form a set, defined as a vicinity of the given input. For any $ \bm{x} \in \mathbb{X} $, its vicinity is expressed as
\begin{equation}
    \mathbb{V}(\bm{x}) = \set{\bm{x'} \mid d(\bm{x}, \bm{x'}) \le \epsilon }
\end{equation}

Since the corresponding distance function $d$ can be a representation of different distance measures, the set notation $\mathbb{V}(\bm{x})$ can also be a representation of various sets, \emph{i.e.}, $\mathbb{V}_1(\bm{x}), \mathbb{V}_2(\bm{x})$ could be $x$-vicinities defined in two different ways.

\paragraph{Function Notation of Vicinity}
The set or distance representation may be inconvenient sometimes~\cite{zhang2024certified}. We may sometimes need the notion of $\mathbf{1}_{\bm{x'}\in\mathbb{V}(\bm{x})}$ to quantify if $\bm{x'}$ is a neighbour of $\bm{x}$. For instance, if we would like to sum the marginal probability of all neighbours of $\bm{x}$, we can $\int_\mathbb{X} \mathbf{1}_{\bm{x'}\in\mathbb{V}(\bm{x})} p(\mathbf{x}=\bm{x'}) d\bm{x'}$ instead of $\int_{\mathbb{V}(\bm{x})} p(\mathbf{x}=\bm{x'}) d\bm{x'}$ to avoid a varying interval of integration.

In this case, a vicinity function, which is an equivalent form of the set, can be defined as
\begin{equation}
\label{eq:vicinity}
    v_{\bm{x}}(\bm{x'}) = \begin{cases}
        \left(\int_{\mathbb{V}(\bm{x})}d\bm{x''}\right)^{-1}, &\text{if}~~ \bm{x'}\in \mathbb{V}(\bm{x})\\
        0, &\text{otherwise}
    \end{cases}
\end{equation}
Essentially, \cref{eq:vicinity} can be viewed as a probability density function (PDF) uniformly defined over the vicinity around an input $\bm{x}$. Now we shift the x-coordinate by $\bm{x}$, we get 
\begin{equation}
    v_{\bm{0}}(\bm{x'} - \bm{x}) = \begin{cases}
        \left(\int_{\mathbb{V}(\bm{0})}d\bm{x''}\right)^{-1}, &\text{if}~~ \bm{x'} - \bm{x}\in \mathbb{V}(\bm{0})\\
        0, &\text{otherwise}
    \end{cases}
\end{equation}
Assuming that the vicinity function is translation invariant, we can drop the subscript $\bm{0}$, and use a positive constant $\epsilon_\mathrm{v}$ to represent $\int_{\mathbb{V}(\bm{0})}d\bm{x''}$, \emph{i.e.}, the volume of the vicinity. Thus, the vicinity function $v: \mathbb{X}\to \set{0, \epsilon_\mathrm{v}^{-1}}$ can be expressed as
\begin{equation}
    v(\bm{x}) = \begin{cases}
        \epsilon_\mathrm{v}^{-1} &\text{if}~~ \bm{x} \in \mathbb{V}(\bm{0})\\
        0, &\text{otherwise}
        \end{cases}
\end{equation}
An example of a one-dimensional input's vicinity is shown in \cref{fig:vicinity1}.

\paragraph{Distribution Notation of Vicinity}

When we need to sample from the vicinity, we need its distribution notation. The  distribution notation for $\bm{x}$-vicinity is $\mathcal{V}(\bm{x})$ whose PDF is denoted as $v: \mathbb{R}^n\to\mathbb{R}$. If $v(\bm{x'}-\bm{x})>0$, we say $\bm{x'}$ is within $\bm{x}$-vicinity.

An imperceptible perturbation from any $\bm{x}$ to $\bm{x'}$ means that $\bm{x'}$ is a `neighbour' of $\bm{x}$, or equivalently, $\bm{x'}$ is in the $\bm{x}$-vicinity. $\bm{x}$-vicinity is a (probabilistic) distribution $\mathcal{V}(\bm{x})$ centred at $\bm{x}$. A standard vicinity $\mathcal{V}(\bm{0})$ is centred at the origin and its PDF is denoted as $v: \mathbb{R}^n\to\mathbb{R}$. Thus, the PDF centred any specific $\bm{x}$ would be $v(\bm{x'}-\bm{x})$.
    
$v$ is typically an even and quasiconcave function. Formally, 
\begin{equation}
\begin{aligned}
    v(\bm{x}) &= v(-\bm{x})\\
    \forall t&\in[0,1]~ \forall \bm{x_1},\bm{x_2}\in \mathbb{R}^n,\quad v\left(t\bm{x_1}+(1-t)\bm{x_2}\right)\ge \min\left(v(\bm{x_1}), v(\bm{x_2})\right)
\end{aligned}
\end{equation}

A uniform $L^p$ $\bm{x}$-vicinity assumes that all inputs outside an $L^p$-norm of $\bm{x}$ are distinguishable from $\bm{x}$ and all inputs within this norm are equally imperceptible from $\bm{x}$. This $L^p$ vicinity function is captured in \cref{eq:uniform_linf} where parameter $\epsilon$ specifies a size.
    \begin{equation}
    \label{eq:uniform_linf}
        v(\bm{x'}-\bm{x}) = \begin{cases}
            \frac{\Gamma(1 + n/p)}{(2\epsilon \Gamma(1 + 1/p))^n}, &\text{if}~~ \norm{\bm{x'}-\bm{x}}_p\le\epsilon\\
            0, &\text{otherwise}
        \end{cases}
    \end{equation}
    The fraction in \cref{eq:uniform_linf} represents the inverse of the $L^p$-norm volume, where $\Gamma$ denotes the gamma function. Vicinity functions assess the likelihood of $\bm{x'}$ being a neighbour to $\bm{x}$. In the uniform $L^p$-norm context, all inputs within the norm are equally valid as neighbours and no inputs outside the norm are neighbours.  
\begin{remark}
    Adversarial examples of an input $\bm{x}$ always reside $\bm{x}$-vicinity.
\end{remark}

\begin{figure}[t]
    \centering
        \includegraphics[width=0.4\linewidth]{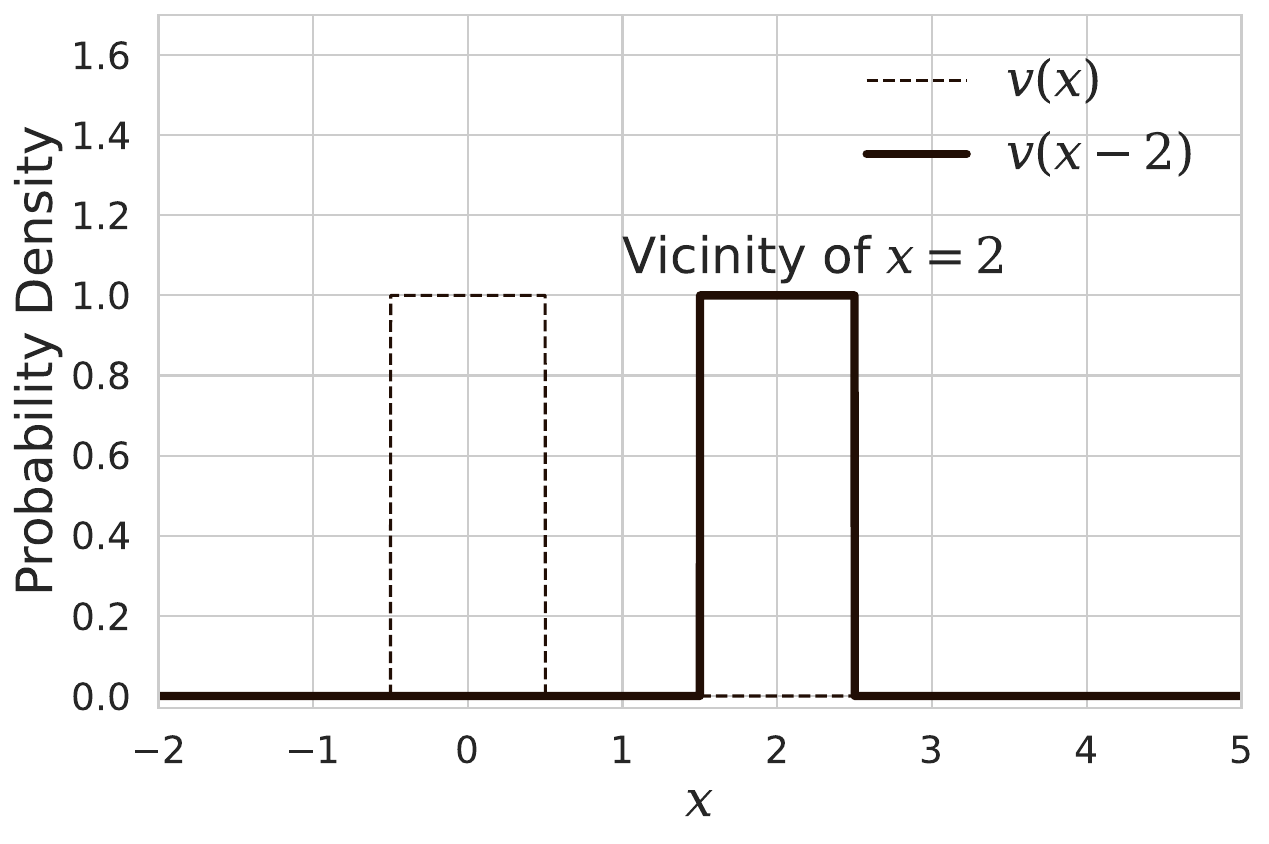}
        \caption{1D visualizations of vicinity function. This vicinity function is a rectangular function that returns a constant value if an input is in the vicinity. Vicinity function $v(\bm{x})$ is shown in dashed line ($\epsilon=0.5$). To get the vicinity at a specific input $\bm{x}=2$, we simply shift $v(\bm{x})$ along the positive direction of the x-axis by 2.}
        \label{fig:vicinity1}
\end{figure}

\section{Complete Proofs and Derivations}
\label{sec:proof}
This section provides detailed proofs for various lemmas, theorems, or corollaries. For each, we restate the original claim followed by a more comprehensive proof than what appears in the main text. Additionally, we include detailed derivations for certain equations not covered in the theorems.

\subsection{Proof of \cref{le:half}}
\label{app:half}
\begin{replemma}{le:half}
For the predictions within the vicinity (of an input $\bm{x}$) to be consistent, at most one class has a prediction probability exceeding $1-\kappa$ in this vicinity. Thus, $\kappa < \frac{1}{2}$.
\end{replemma}

\begin{proof}
    Assume that $\kappa > \frac{1}{2}$ such that any $\sum_{k=0}^{K-1} \mu_k(\bm{x})\,\mathbf{1}_{h(\bm{x}) = k} $ that is greater than or equal to $ \frac{1}{2}$ satisfies the consistency condition because $\frac{1}{2}> 1 - \kappa$. Hence, there may exist some $k_1\neq k_2$ such that $\mu_{k_1} + \mu_{k_2} \le 1$ and
    \begin{equation}
    \label[ineq]{ineq:half}
        \sum_{k=0}^{K-1} \mu_{k_1}(\bm{x})\,\mathbf{1}_{h(\bm{x}) = k_1} = \sum_{k=0}^{K-1} \mu_{k_2}(\bm{x})\,\mathbf{1}_{h(\bm{x}) = k_2} > 1-\kappa.
    \end{equation}
    The existence of such distinct indices $k_1, k_2$ implies that if prediction at $\bm{x}$ is $k_1$, it is consistent with its neighbours. Similarly, if prediction at $\bm{x}$ is $k_2$, it is consistent with its neighbours. However, since $k_1 \neq k_2$, it is not possible for the neighbours' predictions to simultaneously be consistent with both $k_1$ and $k_2$. This scenario contradicts \cref{ineq:half}, and thus the initial assumption does not hold.    
\end{proof}

\subsection{Derivation of the Combined Error Considering probabilistic robustness}
\label{app:combine}

Considering probabilistic robustness, the error at input $\bm{x}$ is a combined error of $e_\text{cor}$ and $e_\text{cns}$ at $\bm{x}$. 
As discussed, we use two intuitions to derive the combined error. First, if $e_\text{cns}(\bm{x}, h; \kappa)=1$, the combined error is always 1. Second, if $e_\text{cns}(\bm{x}, h;\mathcal{V},\kappa)=0$, the combined error equals $e_\text{cor}(\bm{x}, h; P(\textnormal{y}\mid\bm{x}))$. Note that inconsistency is a binary value that takes either 0 or 1 while incorrectness takes a real value from 0 to 1 depending on the posterior at input. In the following, we derive the combined error $e(\bm{x}, h; P(\textnormal{y}\mid\bm{x}), \kappa)$ as expressed in \cref{eq:combine}.
\begin{equation}
\begin{aligned}
    e(\bm{x}, h; P(\textnormal{y}&|\bm{x}), \kappa) = (1 - e_\text{cns}(\bm{x}, h;\mathcal{V},\kappa)) e_\text{cor}(\bm{x}, h; P(\textnormal{y}\mid\bm{x})) + e_\text{cns}(\bm{x}, h;\mathcal{V},\kappa) \\
    &= e_\text{cor}(\bm{x}, h; P(\textnormal{y}\mid\bm{x})) - e_\text{cor}(\bm{x}, h; P(\textnormal{y}\mid\bm{x})) e_\text{cns}(\bm{x}, h;\mathcal{V},\kappa) + e_\text{cns}(\bm{x}, h;\mathcal{V},\kappa) \\
    &= e_\text{cor}(\bm{x}, h; P(\textnormal{y}\mid\bm{x})) + e_\text{cns}(\bm{x}, h;\mathcal{V},\kappa) - e_\text{cor}(\bm{x}, h; P(\textnormal{y}\mid\bm{x})) e_\text{cns}(\bm{x}, h;\mathcal{V},\kappa) \\
    &= 1-(1 - e_\text{cns}(\bm{x}, h;\mathcal{V},\kappa)) (1- e_\text{cor}(\bm{x}, h; P(\textnormal{y}\mid\bm{x})))\\
    &= 1-\left(1-u\left(1- \kappa - \sum_{k=0}^{K-1} \mu_k(\bm{x})\,\mathbf{1}_{h(\bm{x}) = k} \right)\right)\left(\sum_{y=0}^{K-1} P(\textnormal{y}=y\mid \mathbf{x}=\bm{x}) \,\mathbf{1}_{h(\bm{x}) = y} \right)\\
    &= 1-u\left(\kappa - 1 + \sum_{k=0}^{K-1} \mu_k(\bm{x})\,\mathbf{1}_{h(\bm{x}) = k} \right)\left(\sum_{y=0}^{K-1} P(\textnormal{y}=y\mid \mathbf{x}=\bm{x}) \,\mathbf{1}_{h(\bm{x}) = y} \right)
\end{aligned}
\end{equation}

\subsection{Extended Proof of \cref{thm:align}}
\label{app:align}

\begin{reptheorem}{thm:align}
    If $h^*$ is optimal for the probabilistic robustness on a given distribution, \emph{i.e.}, $h^* = \argmin_h \int_{\bm{x} \in \mathbb{R}^n} e(\bm{x})p(\mathbf{x} = \bm{x}) d\bm{x}$, we would always have
    $\forall \bm{x} \in \mathbb{R}^n, h^*(\bm{x}) = \argmax_k \mu_k(\bm{x})$.
\end{reptheorem}

\begin{proof}
    Let $h_1$ and $h_2$ be two distinct classification functions such that $h_1(\bm{x}) = \argmax_k \mu_k(\bm{x})$ and $h_2(\bm{x}) \neq h_1(\bm{x})$. We want to prove $e(\bm{x}, h_1) \le e(\bm{x}, h_2)$, such that $h_1$ must be optimal for probabilistic robustness. First, we denote $k_1, k_2\in \set{0, 1, {\scriptstyle\ldots}, K-1}, k_1 = h_1(\bm{x}), k_2 = h_2(\bm{x})\neq k_1$. Then,
    \begin{equation}
    \begin{aligned}
        e(\bm{x}, h_1) - e(\bm{x}, h_2)&= 1-u\left(\kappa - 1 + \sum_{k=0}^{K-1} \mu_k(\bm{x})\,\mathbf{1}_{h_1(\bm{x}) = k} \right)\left(\sum_{y=0}^{K-1} P(\textnormal{y}=y| \mathbf{x}=\bm{x}) \,\mathbf{1}_{h_1(\bm{x}) = y} \right)\\
        & - 1+u\left(\kappa - 1 + \sum_{k=0}^{K-1} \mu_k(\bm{x})\,\mathbf{1}_{h_2(\bm{x}) = k} \right)\left(\sum_{y=0}^{K-1} P(\textnormal{y}=y| \mathbf{x}=\bm{x}) \,\mathbf{1}_{h_2(\bm{x}) = y} \right)\\
        =u\big(\kappa& - 1 + \mu_{k_2}(\bm{x}) \big)P(\textnormal{y}=k_2\mid \mathbf{x}=\bm{x})  - u\left(\kappa - 1 + \mu_{k_1}(\bm{x}) \right)P(\textnormal{y}=k_1\mid \mathbf{x}=\bm{x}).
    \end{aligned}
    \end{equation}
    Since $\mu_{k_2}(\bm{x}) \le \mu_{k_1}(\bm{x})$, we get $\mu_{k_2}(\bm{x}) \le 1/2$. Recall $\kappa < 1/2$ 
    from \cref{le:half}, we get
    \begin{equation}
        \kappa - 1 + \mu_{k_2}(\bm{x}) ~<~\frac{1}{2} - 1 + \frac{1}{2} = 0
    \end{equation}
    Consequently, when the input of a unit step function is negative, we have $u\left(\kappa - 1 + \mu_{k_2}(\bm{x}) \right) = 0$. Therefore, we get the following expression where the value of a unit step function and the value of conditional probability are both non-negative.
    \begin{equation}
        e(\bm{x}, h_1) - e(\bm{x}, h_2) =  - u\left(\kappa - 1 + \mu_{k_1}(\bm{x}) \right)P(\textnormal{y}=k_1| \mathbf{x}=\bm{x}) \le 0
    \end{equation}
    Hence, the error of $h_1$ is no greater than the error of $h_2$. This inequality applies to any input $\bm{x}$. Hence, a classification function like $h_1$ is optimal.
\end{proof}

\subsection{Proof of \cref{lemma:bound_change}}
\label{app:bound_change}

\begin{replemma}{lemma:bound_change}
A change in $\mu_k$ results from shifting an input by a certain distance $\phi$ within the vicinity. This change is bounded in any direction $\bm{\hat{\phi}}$. Formally,
    \begin{equation}
    \begin{aligned}
        \forall\bm{x} \in \mathbb{R}^n,\forall \phi\in\mathbb{R},        \Bigg(\left(\forall \bm{\hat{\phi}}\in\mathbb{S}^{n-1},~ v\left(\frac{\phi}{2}\bm{\hat{\phi}}\right)>0\right) \to\\
        \forall \bm{\hat{\phi}}\in\mathbb{S}^{n-1},~ \left( \abs{\mu_k(\bm{x} + \phi\bm{\hat{\phi}}) - \mu_k(\bm{x})} \le  \abs{1 - \min_{\bm{\hat{\phi'}}\in\mathbb{S}^{n-1}}\int_{\bm{t}\in\mathbb{R}^n}\min \left( v(\bm{t}  -\phi\bm{\hat{\phi'}}),  v(\bm{t})\right) d\bm{t}}\right)\Bigg),
    \end{aligned}
    \end{equation}
    where $\in\mathbb{S}^{n-1}$ denotes the set of all unit vectors in $\mathbb{R}^n$. 

\end{replemma}

\begin{proof}
    Each $\mu_k$ has a convolutional form as provided in \cref{eq:mu_conv}. Therefore, the change of $\mu_k$ resulting from a shift with magnitude $\phi$ in direction $\phi\bm{\hat{\phi}}$ can be expressed as
    \begin{equation}
        \mu_k(\bm{x} + \phi\bm{\hat{\phi}}) - \mu_k(\bm{x}) = \int_{\bm{t}\in\mathbb{R}^n}\mathbf{1}_{h(\bm{t})=k} ~v(\bm{x} + \phi\bm{\hat{\phi}} - \bm{t})\,d\bm{t} -\int_{\bm{t}\in\mathbb{R}^n}\mathbf{1}_{h(\bm{t})=k} ~v(\bm{x} - \bm{t})\,d\bm{t}
    \end{equation}
    For simplicity, let $\bm{\phi}=\phi\bm{\hat{\phi}}$ for the moment. Then, according to the linearity of integration, we can put the two integrands under the same integral as
    \begin{equation}
        \mu_k(\bm{x} + \bm{\phi}) - \mu_k(\bm{x}) =\int_{\bm{t}\in\mathbb{R}^n}\left(\mathbf{1}_{h(\bm{t})=k} ~v(\bm{x} + \bm{\phi} - \bm{t}) -\mathbf{1}_{h(\bm{t})=k} ~v(\bm{x} - \bm{t})\right)d\bm{t}
    \end{equation}
    Observe that we can combine like terms $\mathbf{1}_{h(\bm{t})=k}$ shared by two parts of the integrands. Thus,
    \begin{equation}
        \mu_k(\bm{x} + \bm{\phi}) - \mu_k(\bm{x}) = \int_{\bm{t}\in\mathbb{R}^n}\mathbf{1}_{h(\bm{t})=k} \left(v(\bm{x} + \bm{\phi} - \bm{t}) - v(\bm{x} - \bm{t})\right)d\bm{t}
    \end{equation}
    Applying the symmetry of the vicinity function about axes, we can get
    \begin{equation}
        \mu_k(\bm{x} + \bm{\phi}) - \mu_k(\bm{x}) = \int_{\bm{t}\in\mathbb{R}^n}\mathbf{1}_{h(\bm{t})=k} \left(v(\bm{t} -\bm{x} -\bm{\phi} ) - v(\bm{t} - \bm{x} )\right)d\bm{t}
    \end{equation}
    Next, we shift the integral limit by $+\bm{x}$, the integrand becomes $\mathbf{1}_{h(\bm{t}+\bm{x})=k}(v(\bm{t} -\bm{\phi} ) - v(\bm{t} ))$, and the interval of integration remains the same. To find the upper and lower bounds of $\mu_k(\bm{x} + \bm{\phi}) - \mu_k(\bm{x})$, we want to find those for this integrand. Observe that $\mathbf{1}_{h(\bm{t}+\bm{x})=k}$ either takes value 0 or 1. Thus, letting $\mathbf{1}_{h(\bm{t}+\bm{x})=k}=1$  if and only if $(v(\bm{t} -\bm{\phi} ) - v(\bm{t} )) > 0$ will maximise the integrand, and letting $\mathbf{1}_{h(\bm{t}+\bm{x})=k}=1$  if and only if $(v(\bm{t} -\bm{\phi} ) - v(\bm{t} )) < 0$ will minimise the integrand. Formally,
    \begin{equation}
        \min \left(0, v(\bm{t} -\bm{\phi} ) - v(\bm{t} )\right) \le \mathbf{1}_{h(\bm{t}+\bm{x})=k}(v(\bm{t} -\bm{\phi} ) - v(\bm{t} )) \le \max \left(0, v(\bm{t} -\bm{\phi} ) - v(\bm{t} )\right)
    \end{equation}    
    Substitute this inequality back into the integral gives rise to
    \begin{equation}
        \abs{\mu_k(\bm{x} + \bm{\phi}) - \mu_k(\bm{x})} \le \int_{\bm{t}\in\mathbb{R}^n}\max\left(0, v(\bm{t} -\bm{\phi} ) - v(\bm{t} )\right)d\bm{t}
    \end{equation}
    Now we add 1 minus 1 to the right-hand side. Note that integrating a probability density function $v$ across the entire domain also equals 1.
    \begin{equation}
    \begin{aligned}
        \text{RHS}
        &= 1 + \int_{\bm{t}\in\mathbb{R}^n}\max\left(0,  v(\bm{t} -\bm{{\phi}} ) - v(\bm{t} )\right)d\bm{t} - 1\\
        &= 1 + \int_{\bm{t}\in\mathbb{R}^n}\left(\max\left(0, v(\bm{t} -\bm{{\phi}} ) - v(\bm{t} )\right)-v(\bm{t} -\bm{{\phi}} )\right)d\bm{t} \\
        &= 1 + \int_{\bm{t}\in\mathbb{R}^n}\max\left(-v(\bm{t} -\bm{{\phi}} ),- v(\bm{t} )\right)d\bm{t} \\
        &= 1 - \int_{\bm{t}\in\mathbb{R}^n}\min\left(v(\bm{t} -\bm{{\phi}} ), v(\bm{t} )\right)d\bm{t} \\
    \end{aligned}
    \end{equation}
    Hence, the difference between $\mu_k(\bm{x} + \bm{\phi})$ and $ \mu_k(\bm{x})$ is bounded by the complement of a vicinity from another vicinity shifted by $\bm{\phi}$.

    Further, recall $\bm{\phi}=\phi\bm{\hat{\phi}}$, the integrand thus becomes $\min(v(\bm{t} -\phi\bm{\hat{\phi}} ), v(\bm{t} ))$, and
    \begin{equation}
        \min\left(v(\bm{t} -\phi\bm{\hat{\phi}} ), v(\bm{t} )\right) \le \min\left(v(\bm{t} ), \min_{\bm{\hat{\phi}}\,\in\,\mathbb{S}^{n-1}} v(\bm{t} -\phi\bm{\hat{\phi}} )\right)
    \end{equation}

    As a result, the problem of maximising the change in $\mu_k$ by a shifting magnitude $\phi$ is converted into the optimisation of finding the direction that results in the minimum overlap between the original vicinity function ($v$) and the vicinity function shifted by $\phi$ in that direction. The resulting upper bound can be expressed as \cref{eq:bijective}, which is re-displayed as follows.
    \begin{equation}
        \mathop{\Delta}\limits_{\max}\mu_k (\phi)
        = \max_{\bm{\hat{\phi}}\,\in\,\mathbb{S}^{n-1}} \mu_k(\bm{x} + \phi\bm{\hat{\phi}}) - \mu_k(\bm{x})
        = 1 - \min_{\bm{\hat{\phi}}\,\in\,\mathbb{S}^{n-1}}\int_{\bm{t}\in\mathbb{R}^n}\min \left( v(\bm{t}  -\phi\bm{\hat{\phi}}),  v(\bm{t})\right) d\bm{t}
    \end{equation}

    Similarly, the lower bound is the negative of the upper bound. Also, as long as two vicinities overlap, the change of $\mu_k$ is less than 1.
\end{proof}

\subsection{Extented Proof of \cref{thm:continuous}}
\label{app:continuous}

In the original proof of \cref{thm:continuous}, a value $1/2-\kappa$ is involved. Here, we explain how we get this value, and why it stands for the minimum required $\mu_k$ drop to allow an adversarial example.

\begin{proof}[Why $1/2-\kappa$ marks the minimum $\mu_k$ change to have an adversarial example]
  For any consistent input $\bm{x}$, we have $\mu_{k^*}(\bm{x}) \ge 1 - \kappa$, where $\kappa^*$ is the major prediction in $\bm{x}$-vicinity. Consider $\bm{x'}$ as a neighbour of $\bm{x}$. If $\mu_{k^*}(\bm{x'}) > 1/2$, then we know that $h(\bm{x'})=k^*$ according to \cref{thm:align}, \emph{i.e.}, $\bm{x'}$ has the same prediction as $\bm{x}$ does. In this way $\bm{x'}$ can be possibly an adversarial example of $\bm{x}$ only if $\mu_{k^*}(\bm{x'}) \le 1/2$. Thus, we get the minimum requirement of $\mu_k$ drop to allow an adversarial example to appear to be $1 - \kappa - 1/2 = 1/2-\kappa$.
\end{proof}

\subsection{Proof of \cref{cor:slope}}
\label{app:slope}

\begin{repcorollary}{cor:slope}

    There exists a finite real value such that for all inputs, the directional derivative value with respect to any arbitrary nonzero vector $\bm{\hat{\phi}}$ (unit vector) does not exceed this finite value and does not fall below the negative of this value. Formally, $\cdot$ denotes the dot product, and
    \begin{equation}
        \exists b\in\mathbb{R}, \forall \bm{x}\in \mathbb{R}^n, \forall\bm{\hat{\phi}}\in\mathbb{S}^{n-1},~  -b \leq \nabla \mu_k(\bm{x} )\cdot \bm{\hat{\phi}} \leq b.
    \end{equation}
\end{repcorollary}

\begin{proof}
    The directional derivative of $\mu_k$ in the direction of $\bm{\hat{\phi}}\in\mathbb{S}^{n-1}$ can be expressed as
    \begin{equation}
        \nabla \mu_k(\bm{x} )\cdot \bm{\hat{\phi}}
        =\lim _{\delta \to 0} \frac{\mu_k(\bm{x} +\delta \bm{\hat{\phi}} )-\mu_k(\bm{x} )}{\delta}
    \end{equation}
    According to \cref{lemma:bound_change}, we can re-express the numerator such that the directional derivative is 
    \begin{equation}
    \begin{aligned}
        \nabla \mu_k(\bm{x} )\cdot  \bm{\hat{\phi}}
        &=\lim _{\delta \to 0}\frac{\int_{\bm{t}\in\mathbb{R}^n}\mathbf{1}_{h(\bm{t}+\bm{x})=k} \left(v(\bm{t}  -\delta\bm{\hat{\phi}} ) - v(\bm{t})\right)d\bm{t}}{\delta}\\
        &=\lim _{\delta \to 0}\int_{\bm{t}\in\mathbb{R}^n}\frac{\mathbf{1}_{h(\bm{t}+\bm{x})=k} \left(v(\bm{t}  -\delta\bm{\hat{\phi}} ) - v(\bm{t})\right)}{\delta}d\bm{t}\\
        &=\lim _{\delta \to 0}\int_{\bm{t}\in\mathbb{R}^n}\mathbf{1}_{h(\bm{t}+\bm{x})=k}\frac{ v(\bm{t}  -\delta\bm{\hat{\phi}} ) - v(\bm{t})}{\delta}d\bm{t}\\
        &= \int_{\bm{t}\in\mathbb{R}^n}\mathbf{1}_{h(\bm{t}+\bm{x})=k}\lim _{\delta \to 0}\frac{ v(\bm{t}  -\delta\bm{\hat{\phi}} ) - v(\bm{t})}{\delta}d\bm{t}\\
        &=-\int_{\bm{t}\in\mathbb{R}^n}\mathbf{1}_{h(\bm{t}+\bm{x})=k}\lim _{\delta \to 0}\frac{ v(\bm{t}) - v(\bm{t}  -\delta\bm{\hat{\phi}} )  }{\delta}d\bm{t}\\ 
        &=-\int_{\bm{t}\in\mathbb{R}^n}\mathbf{1}_{h(\bm{t})=k}\lim _{\delta \to 0}\frac{ v(\bm{t}-\bm{x}) - v(\bm{t} -\bm{x} -\delta\bm{\hat{\phi}} )  }{\delta}d\bm{t}\\ 
        &=\int_{\bm{t}\in\mathbb{R}^n}\mathbf{1}_{h(-\bm{t}+\bm{x})=k}\lim _{\delta \to 0}\frac{ v(-\bm{t}) - v(-\bm{t}  -\delta\bm{\hat{\phi}} )  }{\delta}d\bm{t}\\ 
        &=-\int_{\bm{t}\in\mathbb{R}^n}\mathbf{1}_{h(-\bm{t}+\bm{x})=k}\lim _{\delta \to 0}\frac{ v(\bm{t}  +\delta\bm{\hat{\phi}} ) - v(\bm{t})}{\delta}d\bm{t}\\
        &=-\int_{\bm{t}\in\mathbb{R}^n}\mathbf{1}_{h(\bm{t}+\bm{x})=k}\left(\nabla v(\bm{t} )\cdot \bm{\hat{\phi}}\right)d\bm{t}\\
    \end{aligned}
    \end{equation}    
    The directional derivative of $\mu_k$ is maximised when the binary function takes 1 if and only if $\nabla v(\bm{t} )\cdot \bm {\phi} < 0$. Also, $v$ is even in every dimension, and when $\nabla v(\bm{t} )\cdot \bm {\phi} < 0$, it is necessary that if $\nabla v(-\bm{t} )\cdot \bm {\phi} > 0$. Thus, the direction that maximises the directional derivative of $\mu_k$ can be expressed as
    \begin{equation}
    \begin{aligned}
        \bm{\phi}^*
        &=-\begin{bmatrix}
           \int_{\bm{t}\in\mathbb{R}^n}\abs{\frac{\partial v(\bm{t} )}{\partial t_1}}d\bm{t} \\
           \int_{\bm{t}\in\mathbb{R}^n}\abs{\frac{\partial v(\bm{t} )}{\partial t_2}}d\bm{t} \\
           \vdots \\
           \int_{\bm{t}\in\mathbb{R}^n}\abs{\frac{\partial v(\bm{t} )}{\partial t_n}}d\bm{t}
         \end{bmatrix}
    \end{aligned}
    \end{equation}
    Thus, the upper bound of the directional derivative of $\mu_k$ can be written as
    \begin{equation}
    \begin{aligned}
        \nabla \mu_k(\bm{x} )\cdot \frac {\bm {\phi} }{\abs{\bm {\phi}}}
        &\le\frac{1}{2}\int_{\bm{t}\in\mathbb{R}^n}\abs{\nabla v(\bm{t} )\cdot \frac {\bm {\phi}^* }{\abs{\bm {\phi}^*}}}d\bm{t}\\
        &= \frac{1}{2\abs{\bm {\phi}^*}}\int_{\bm{t}\in\mathbb{R}^n}\abs{\sum_{i=1}^n \left(\frac{\partial v(\bm{t} )}{\partial t_i} \int_{\bm{\tau}\in\mathbb{R}^n}\abs{\frac{\partial v(\bm{\tau} )}{\partial \tau_i}}d\bm{\tau} \right)}d\bm{t}\\
    \end{aligned}
    \end{equation}
    Since function $v$ is not a delta distribution PDF, this bound is always a finite number.    
\end{proof}

\subsection{Example of \cref{cor:slope}}
\label{app:uniform1d}

\begin{example}
\label{ex:uniform1d}
    \begin{equation}
    \begin{aligned}
        &= \frac{1}{2\abs{\bm {\phi}^*}}\int_{\bm{t}\in\mathbb{R}^n}\abs{\sum_{i=1}^n \left(\frac{\partial v(\bm{t} )}{\partial t_i} \int_{\bm{\tau}\in\mathbb{R}^n}\abs{\frac{\partial v(\bm{\tau} )}{\partial \tau_i}}d\bm{\tau} \right)}d\bm{t}.\\
        &= \frac{1}{2}\int_{\bm{t}\in\mathbb{R}^n}\abs{\sum_{i=1}^n \left(\frac{\partial v(\bm{t} )}{\partial t_i}  \right)}d\bm{t}.\\ 
        &= \frac{1}{2}\int_{\bm{t}\in\mathbb{R}^n}\abs{ \left(\frac{\partial v(\bm{t} )}{\partial t_0}  \right)}d\bm{t}.\\
        &= \frac{1}{2}\int_{\bm{t}\in\mathbb{R}^n}\abs{ \left(\frac{\partial v(\bm{t} )}{\partial t_0}  \right)}d\bm{t}.
    \end{aligned}
    \end{equation}

    Suppose we have input $\bm{x}\in\mathbb{R}$, $v:\mathbb{R}\to\mathbb{R}$, and $v$ is a symmetric uniform distribution function. According to \cref{cor:slope}, the slope of $\mu_k$ in this example is within $\pm \frac{1}{2\epsilon}$. We can validate this value in \cref{eq:leibniz} using Leibniz's rule for differentiation under the integral sign.
    \begin{equation}
    \label{eq:leibniz}
       \frac{d (\mathbf{1}_{h()=k} * v)(\bm{x})}{d\bm{x}} = \frac{1}{2\epsilon}\frac{d}{d\bm{x}} \int_{-\epsilon}^{\epsilon} \mathbf{1}_{h(\bm{x} - \bm{t})=k} \,d\bm{t} = \frac{1}{2\epsilon} \big( \mathbf{1}_{h(\bm{x} - \epsilon)=k} - \mathbf{1}_{h(\bm{x} + \epsilon)=k} \big)
    \end{equation}
    The intuition of this example is that when a vicinity shifts from a region where all samples are labelled with one class to a region where all samples are labelled with another class, the slope reaches its maximum.
\end{example}

\subsection{Proof of \cref{cor:linf}}
\label{app:linf}
\begin{repcorollary}{cor:linf}
    If $h^*$ is optimal for the probabilistic robustness with respect to an $L^\infty$-vicinity on a given distribution, then the vicninty size for the deterministically robust region around each consistent is $\epsilon\left(1 - (2\kappa)^{\frac{1}{n}}\right)$.
\end{repcorollary}

\begin{proof}
    An $L^\infty$ norm looks like a $n$-dimensional cube. A two-dimensional illustration is given in \cref{fig:2d_square_example}. Generally, the vicinity function can be expressed as
    \begin{equation}
    \label{eq:linf}
        v(\bm{x}) =\begin{cases}{(2\epsilon)^{-n}}&{\text{for }} \norm{\bm{x}}_{\infty} \le \epsilon\\0&{\text{otherwise.}}\end{cases}
    \end{equation}
    According to \cref{thm:continuous}, if we would like to find the closest adversarial example to a consistent input at the upper right centre of the yellow vicinity, we first need to find a shift magnitude that causes as large as a $\mu_{k^*}$ drop by $1/2 - \kappa$. Suppose this drop goes further and the closest (probabilistucally) consistent input is found. In this way, the minimum $\mu_k$ drop from consistent point $\bm{x}$ and its (probabilistucally) consistent adversarial example is $1-\kappa - \kappa = 1 - 2\kappa$  In an $L^\infty$ vicinity scenario, the magnitude of shift can be solved based on \cref{eq:linfsolve}.

    \begin{equation}
    \label{eq:linfsolve}
        1 - 2\kappa = 1-\min_{\bm{\hat{\phi}}} \prod_{i=1}^{n} \max(0, 2\epsilon - \phi\hat{\phi}_i)
    \end{equation}
    where $\hat{\phi}_i$ is each element of the unit directional vector $\bm{\hat{\phi}}$. Since the shift is within the vicinity, we can write $\max(0, 2\epsilon - \phi\hat{\phi}_i)$ as $2\epsilon - \phi\hat{\phi}_i$. For two identical n-dimensional cubes, the fastest way to reduce the overlap is to move one of them in the diagonal direction, such that $\hat{\phi_i} = \frac{1}{\sqrt{n}}$. Then the overlap volume becomes $(2\epsilon - \phi/\sqrt{n})^n$. Then, \cref{eq:linfsolve} can be simplified as $2\kappa (2\epsilon)^n = (2\epsilon - \abs{\bm{\phi}}/\sqrt{n})^n$. Solving this equation, we get $\epsilon(1 - (2\kappa)^{\frac{1}{n}})$. This $\epsilon(1 - (2\kappa)^{\frac{1}{n}})$ also serves as the vicinity size for deterministic robustness at this input. As $n$ grows, this vicinity size decreases. As $\kappa$ grows, this vicinity size may grow.
\end{proof}

We may visualise the effect of \cref{cor:linf} using a two-dimensional example in \cref{fig:2d_square_example}. Let $\bm{x}, \bm{x'}$ be two probabilistically consistent inputs and $h(\bm{x})\neq h(\bm{x'})$. Suppose $\bm{x'}$ is the nearest adversarial example (of $\bm{x}$) that achieves its own probabilistic consistency. Thus, $\bm{x'}$ must be in the direction (that travels away from $\bm{x}$) that fastest decreases the vicinity overlap (suggested by \cref{thm:continuous}), \emph{i.e.}, the diagonal (suggested by \cref{cor:linf}). Consequently, the nearest adversarial example of $\bm{x}$ would occur on the midpoint between $\bm{x}$ and $\bm{x'}$. The shift from $\bm{x}$ to $\bm{x'}$ is $-2(\phi_1\bm{\hat{x_1}} + \phi_2\bm{\hat{x_2}})$. Each triangle accounts for a $\kappa$ portion of the original vicinity volume, and the vicinity overlap is $2\kappa$. The dashed box $\mathbb{V}^{\downarrow\kappa}$ has side length $2\phi_i$. Thus, solving $(2\epsilon - 2\phi_i)^2 = 2\kappa(2\epsilon)^2$, we get $\mathbb{V}^{\downarrow\kappa}$ has vicinity size $\phi_1 = \phi_2 = (1 - \sqrt{2\kappa})\epsilon$. Although \cref{cor:linf} specifically captures the $L^{\infty}$ scenario, we remark that other vicinity types can be analysed similarly. First, the direction that fastest decreases the vicinity overlap needs to be found. Then, the distance between the input and its nearest adversarial example can be measured.

\subsection{Extended Proof of \cref{thm:triplet}}
\label{app:triplet}
\begin{reptheorem}{thm:triplet}
    The upper bound of probabilistic robust accuracy monotonically increases as $\kappa$ grows. Further, for all tolerance levels $\kappa$, the upper bound of probabilistic robust accuracy lies between the upper bound of deterministic robust accuracy and vanilla accuracy. Formally,
    \begin{equation}
    \begin{aligned}
        \forall\kappa_1,\kappa_2.&\quad (\kappa_1 <\kappa_2)\to \min _{h}\Upsilon^+_{\textnormal{prob}}(D, h, \mathcal{V}, \kappa_1) \le \min _{h}\Upsilon^+_{\textnormal{prob}}(D, h, \mathcal{V}, \kappa_2) \\
        \forall\kappa.& \quad \min _{h} \Upsilon^+_{\textnormal{rob}}(D, h, \mathbb{V}) \le \min _{h}\Upsilon^+_{\textnormal{prob}}(D, h, \mathcal{V}, \kappa) \le \min _{h} \Upsilon^+_{\textnormal{acc}}(D, h)
    \end{aligned}
    \end{equation}    
\end{reptheorem}
\begin{proof}
    Let $\kappa_1<\kappa_2$, such that for $\kappa_1, \kappa_2$, we have their corresponding error as the following equation.
    \begin{equation}
    \begin{aligned}
        e(\bm{x}, \kappa_1) - e(\bm{x}, \kappa_2)&= 1-u\left(\kappa_1 - 1 + \sum_{k=0}^{K-1} \mu_k(\bm{x})\,\mathbf{1}_{h(\bm{x}) = k} \right)\left(\sum_{y=0}^{K-1} P(\textnormal{y}=y| \mathbf{x}=\bm{x}) \,\mathbf{1}_{h(\bm{x}) = y} \right)\\
        & - 1+u\left(\kappa_2 - 1 + \sum_{k=0}^{K-1} \mu_k(\bm{x})\,\mathbf{1}_{h(\bm{x}) = k} \right)\left(\sum_{y=0}^{K-1} P(\textnormal{y}=y| \mathbf{x}=\bm{x}) \,\mathbf{1}_{h(\bm{x}) = y} \right)\\
    \end{aligned}
    \end{equation}
    Note the cancelled 1 and the like terms in the rest of the terms, we can further write the above equation as \cref{eq:monokappa}, where we let $T = \sum_{k=0}^{K-1} \mu_k(\bm{x})\,\mathbf{1}_{h(\bm{x}) = k}$. Note that $T$ is just a temporary substituting variable, and we do not mean to use is to denote any particular quantity.
    \begin{equation}
    \label{eq:monokappa}
        e(\bm{x}, \kappa_1) - e(\bm{x}, \kappa_2)=\left(\sum_{y=0}^{K-1} P(\textnormal{y}=y| \mathbf{x}=\bm{x}) \,\mathbf{1}_{h(\bm{x}) = y} \right)\big(
        u\left(\kappa_2 - 1 + T \right)-u\left(\kappa_1 - 1 +  T\right)\big)
    \end{equation}
    Now that $e(\bm{x}, \kappa_1) - e(\bm{x}, \kappa_2)$ is a product of two expressions. Since the posterior probability $P(\textnormal{y}\mid \mathbf{x}=\bm{x})$ is non-negative, the sum of posteriors, \emph{i.e.}, the former expression, is non-negative. Thus, the sign of $e(\bm{x}, \kappa_1) - e(\bm{x}, \kappa_2)$ is the same as $u\big(\kappa_2 - 1 + T \big)  - u\left(\kappa_1 - 1 + T\right)$.
    
    Since $\kappa_1<\kappa_2$, we get $\kappa_1 - 1 + T < \kappa_2 - 1 + T$. Further, the unit step function ($u$) is monotonically increasing, we get $e(\bm{x}, \kappa_1) - e(\bm{x}, \kappa_2)\ge 0$. A more stringent $\kappa$ (\emph{i.e.}, smaller) leads to a lower or equal upper bound of probabilistic robust accuracy. For deterministic robust accuracy, $\kappa=0$, which is the least value. Thus, deterministic robust accuracy has a lower upper bound than probabilistic robust accuracy does.
    

    On the other hand, from the intuition of error combination in \cref{sec:modelling}, we get that the combined error is $e(\bm{x}) = 1-(1 - e_\text{cns}(\bm{x}, h;\mathcal{V},\kappa)) (1- e_\text{cor}(\bm{x}, h; P(\textnormal{y}\mid\bm{x})))$. As $0\le e_\text{cns}\le 1$, we get $e(\bm{x}) \ge  1-(1 - e_\text{cor}(\bm{x}, h;\kappa))$. Note that the expectation of $e_\text{cor}(\bm{x}, h;\kappa)$ is the error in accuracy. Thus, $\Upsilon^+_\text{prob}(D, h, \mathcal{V}, \kappa) \le\Upsilon^+_\text{acc}(D, h)$. So it is with their upper bounds. 
\end{proof}

\section{Additional Experiments and Plots}
In this section, we present the results of some additional experiments in which we investigate the
effect of the upper bounds and decision rules.

\subsection{Setup Details}
\label{app:setup}

The experiments are conducted with four data sets: two synthetic ones (\emph{i.e.}, Moons and Chan~\cite{chen2023evaluating}, whose distributions are illustrated in \cref{fig:moonchan}) and two standard benchmarks (\emph{i.e.}, FashionMNIST~\cite{xiao2017fashion} and CIFAR-10~\cite{krizhevsky2009learning}). Moons is used for binary classification with two-dimensional features, where each class's distribution is described analytically with specific likelihood equations, and uses a three-layer Multi-Layer Perceptron (MLP) neural network for classification. The Chan data set, also for binary classification with two-dimensional features, differs in that it does not follow a standard PDF pattern, requiring kernel density estimation (KDE) for non-parametric PDF estimation, and also uses the three-layer MLP. FashionMNIST, a collection of fashion item images, involves a 10-class classification task with 784-dimensional inputs (28$\times$28 pixel grayscale images). Each class has an equal prior probability, and their conditional distributions are estimated non-parametrically using KDE. CIFAR-10 uses images with a resolution of 32$\times$32 pixels. Similar to FashionMNIST, it has a balanced class distribution and is estimated using KDE. We use a seven-layer convolutional neural network (CNN-7)~\cite{shi2021fast} as the classifier of both FashionMNIST and CIFAR-10. We adopt a direct approach~\cite{ishida2022performance, zhang2024certified} to compute the original Bayes error and deterministic Bayes error of both real-world data sets~\cite{ishida2022performance}.
\begin{figure}[t]
  \centering
  \begin{subfigure}[t]{0.23\linewidth}
    \centering
    \captionsetup{justification=centering}
    \includegraphics[width=\linewidth]{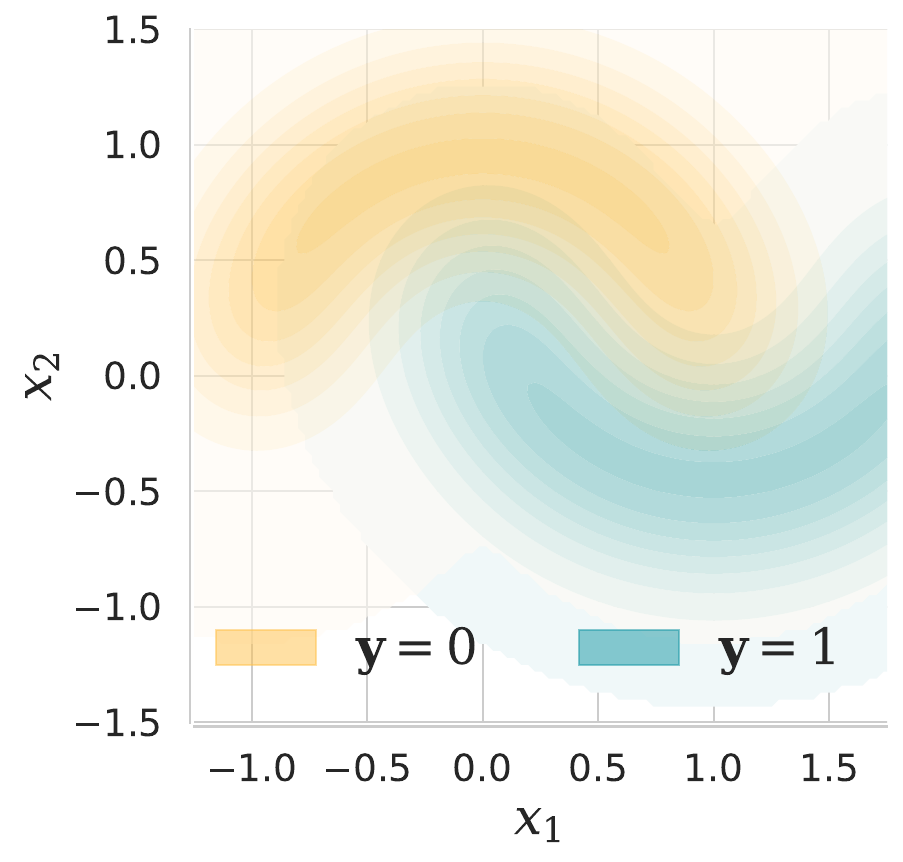}
    \caption{Moons}
  \end{subfigure}%
  \begin{subfigure}[t]{0.23\linewidth}
    \centering
    \captionsetup{justification=centering}
    \includegraphics[width=\linewidth]{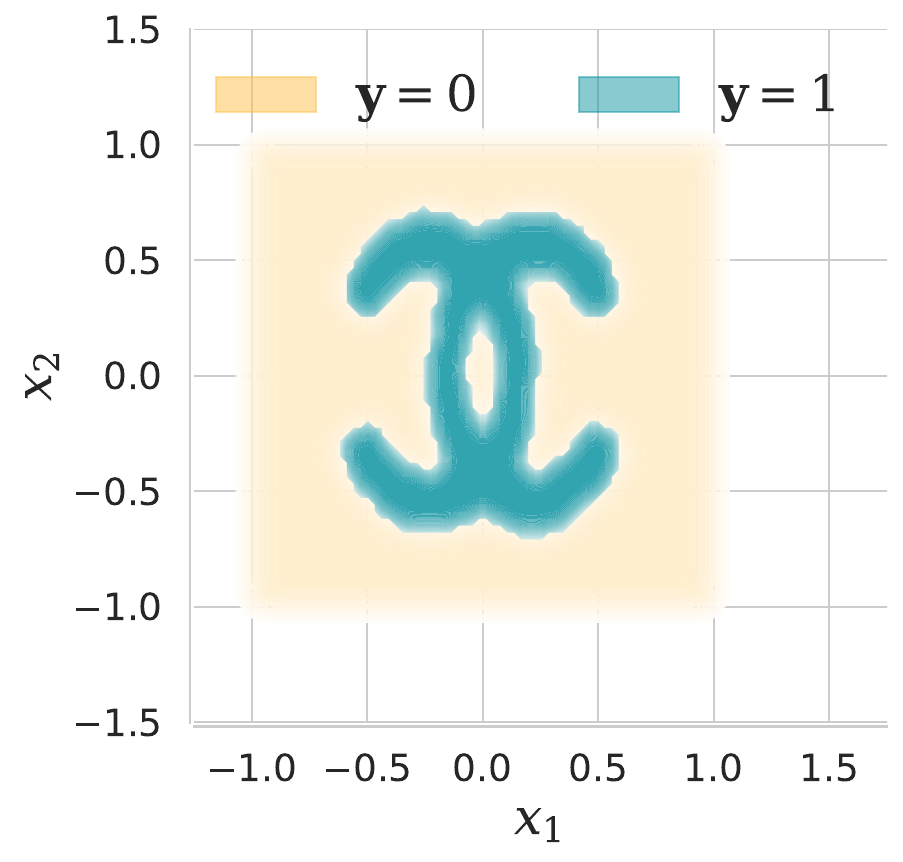}
    \caption{Chan}
  \end{subfigure}%
  \begin{subfigure}[t]{0.07\linewidth}
    \centering
    \includegraphics[width=\linewidth, trim=2.cm -2.5cm 0.32cm 0.cm, clip]{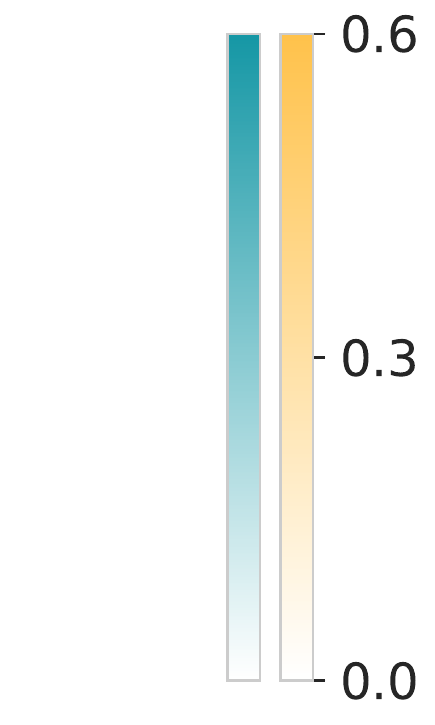}
  \end{subfigure}
  \caption{This figure illustrates the conditional distribution for (a) Moons~\cite{scikit-learn} and (b) Chan~\cite{chen2023evaluating}.
  \label{fig:moonchan}}
\end{figure}

\subsection{Statistical Significance of Experiments}
\label{app:sig}

We include confidence levels of performances of trained classifiers~\cref{tab:sig_before} and their corresponding voting classifiers~\cref{tab:sig_after}. We provide their statistical significance to support the claims addressed by our research questions.

Besides, the train/validation splits follow that of the data set~\cite{xiao2017fashion,krizhevsky2009learning} if there is already a split guideline. For Moons~\cite{scikit-learn} and Chan~\cite{chen2023evaluating}, we follow the setup in \cite{zhang2024certified}.

\begin{table}[t]
  \caption{Probabilistic robustness of classifiers before voting. The value in parentheses represents the 95\% confidence level range when we repeat the same program 100 times.}
  \label{tab:sig_before}
  \centering
  \resizebox{0.7\textwidth}{!}{
  \begin{tabular}{lcccc}
    \toprule
     &Moons &Chan &FashionMNIST &CIFAR-10  \\
    \midrule
    DA~\cite{shorten2019survey} & 85.35 ($\pm$0.020\%) & 67.96 ($\pm$0.023\%) & 84.12 ($\pm$0.000\%) & 76.07  ($\pm$0.021\%) \\
    \midrule
    RS~\cite{cohen2019certified} &  84.76 ($\pm$0.020\%)  & 64.67 ($\pm$0.002\%) & 86.29 ($\pm$0.012\%) & 87.98  ($\pm$0.038\%) \\
    \midrule
    CVaR~\cite{robey2022probabilistically} & 85.52 ($\pm$0.017\%)& 69.46 ($\pm$0.034\%) & 88.50 ($\pm$0.028\%) & 90.63 ($\pm$0.007\%) \\
    \bottomrule
  \end{tabular}}
\end{table}

\begin{table}[t]
  \caption{Probabilistic robustness of classifiers after voting. The value in parentheses represents the 95\% confidence level range when we repeat the same program 100 times.}
  \label{tab:sig_after}
  \centering
  \resizebox{0.7\textwidth}{!}{
  \begin{tabular}{lcccc}
    \toprule
     &Moons &Chan &FashionMNIST &CIFAR-10  \\
    \midrule
    DA~\cite{shorten2019survey} & 85.60 ($\pm$0.034\%) & 68.86 ($\pm$0.006\%) & 87.48 ($\pm$0.022\%) & 81.38 ($\pm$0.026\%)\\
    \midrule
    RS~\cite{cohen2019certified} & 85.18 ($\pm$0.005\%) & 66.77 ($\pm$0.008\%) & 88.13 ($\pm$0.037\%) & 88.95 ($\pm$0.008\%)\\
    \midrule
    CVaR~\cite{robey2022probabilistically} & 85.66 ($\pm$0.011\%)& 70.05 ($\pm$0.005\%) & 91.07 ($\pm$0.028\%) & 90.77 ($\pm$0.011\%)\\
    \bottomrule
  \end{tabular}}
\end{table}

\subsection{How does the sample size affect the voting effectiveness?}
\label{app:rq1}

\begin{figure}[t]
    \centering
    \includegraphics[width=0.3\linewidth]{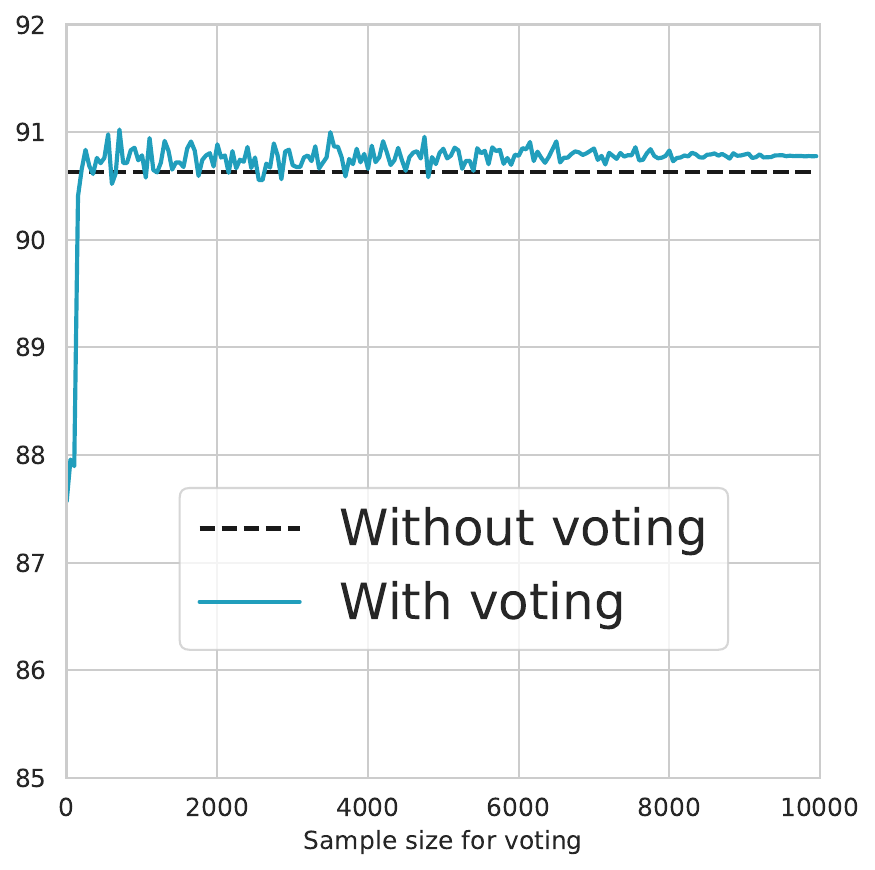}
    \caption{Probabilistic robust accuracy of voting classifier (CVaR~\cite{robey2022probabilistically} tested on CIRAR-10~\cite{krizhevsky2009learning}) as the sample size grows to 10,000.}
    \label{fig:sample}
\end{figure}
The voting process of classifier $h^\dag$ can be viewed as a process of taking the expected value of prediction within the vicinity. The law of large numbers is key to understanding the relationship between sample size and expectation. It states that as the sample size increases, the sample mean converges to the true expectation.

To validate this effect empirically, as well as to verify a suitable sample size at which the performance of the voting classifiers can be properly represented, we gradually increase the sample size from 10 to 10,000. As demonstrated in \cref{fig:sample}, at a small sample size, the voting may result in an increase or a decrease in the performance. However, as the sample exceeds 100, its positive impact on the probabilistic robust accuracy becomes more noticeable. Thus, our empirical intuition matches \cref{thm:align}.



\begin{thebibliography}{10}

\bibitem{athalye2018synthesizing}
Anish Athalye, Logan Engstrom, Andrew Ilyas, and Kevin Kwok.
\newblock Synthesizing robust adversarial examples.
\newblock In Jennifer Dy and Andreas Krause, editors, {\em Proceedings of the 35th International Conference on Machine Learning}, volume~80 of {\em Proceedings of Machine Learning Research}, pages 284--293. PMLR, 10--15 Jul 2018.

\bibitem{balunovic2020adversarial}
Mislav Balunovic and Martin~T. Vechev.
\newblock Adversarial training and provable defenses: Bridging the gap.
\newblock In {\em 8th International Conference on Learning Representations, {ICLR} 2020, Addis Ababa, Ethiopia, April 26-30, 2020}. OpenReview.net, 2020.

\bibitem{bassett2019maximum}
Robert Bassett and Julio Deride.
\newblock Maximum a posteriori estimators as a limit of bayes estimators.
\newblock {\em Mathematical Programming}, 174(1):129--144, Mar 2019.

\bibitem{berisha2015empirically}
Visar Berisha, Alan Wisler, Alfred~O. Hero, and Andreas Spanias.
\newblock Empirically estimable classification bounds based on a nonparametric divergence measure.
\newblock {\em IEEE Transactions on Signal Processing}, 64(3):580--591, 2016.

\bibitem{bhattacharya2019survey}
Saswati Bhattacharya and Mousumi Gupta.
\newblock A survey on: Facial emotion recognition invariant to pose, illumination and age.
\newblock In {\em 2019 Second International Conference on Advanced Computational and Communication Paradigms (ICACCP)}, pages 1--6. IEEE, 2019.

\bibitem{chen2023evaluating}
Qingqiang Chen, Fuyuan Cao, Ying Xing, and Jiye Liang.
\newblock Evaluating classification model against bayes error rate.
\newblock {\em IEEE Transactions on Pattern Analysis and Machine Intelligence}, 45(8):9639--9653, 2023.

\bibitem{chiang2020certified}
Ping{-}yeh Chiang, Renkun Ni, Ahmed Abdelkader, Chen Zhu, Christoph Studer, and Tom Goldstein.
\newblock Certified defenses for adversarial patches.
\newblock In {\em 8th International Conference on Learning Representations, {ICLR} 2020, Addis Ababa, Ethiopia, April 26-30, 2020}. OpenReview.net, 2020.

\bibitem{cohen2019certified}
Jeremy Cohen, Elan Rosenfeld, and Zico Kolter.
\newblock Certified adversarial robustness via randomized smoothing.
\newblock In Kamalika Chaudhuri and Ruslan Salakhutdinov, editors, {\em Proceedings of the 36th International Conference on Machine Learning}, volume~97 of {\em Proceedings of Machine Learning Research}, pages 1310--1320. PMLR, PMLR, 09--15 Jun 2019.

\bibitem{fukunaga1975k}
K.~Fukunaga and L.~Hostetler.
\newblock k-nearest-neighbor bayes-risk estimation.
\newblock {\em IEEE Transactions on Information Theory}, 21(3):285--293, 1975.

\bibitem{fukunaga1990introduction}
Keinosuke Fukunaga.
\newblock {\em Introduction to Statistical Pattern Recognition (2nd Ed.)}.
\newblock Academic Press Professional, Inc., USA, 1990.

\bibitem{ganin2016domain}
Yaroslav Ganin, Evgeniya Ustinova, Hana Ajakan, Pascal Germain, Hugo Larochelle, Fran{\c{c}}ois Laviolette, Mario March, and Victor Lempitsky.
\newblock Domain-adversarial training of neural networks.
\newblock {\em Journal of Machine Learning Research}, 17(59):1--35, 2016.

\bibitem{garber1988bounds}
F.D. Garber and A.~Djouadi.
\newblock Bounds on the bayes classification error based on pairwise risk functions.
\newblock {\em IEEE Transactions on Pattern Analysis and Machine Intelligence}, 10(2):281--288, 1988.

\bibitem{goodfellow2014explaining}
Ian~J. Goodfellow, Jonathon Shlens, and Christian Szegedy.
\newblock Explaining and harnessing adversarial examples.
\newblock In Yoshua Bengio and Yann LeCun, editors, {\em 3rd International Conference on Learning Representations, {ICLR} 2015, San Diego, CA, USA, May 7-9, 2015, Conference Track Proceedings}, 2015.

\bibitem{guerin2021certifying}
Joris Guerin, Kevin Delmas, and Jérémie Guiochet.
\newblock Certifying emergency landing for safe urban uav.
\newblock In {\em 2021 51st Annual IEEE/IFIP International Conference on Dependable Systems and Networks Workshops (DSN-W)}, pages 55--62, 2021.

\bibitem{ishida2022performance}
Takashi Ishida, Ikko Yamane, Nontawat Charoenphakdee, Gang Niu, and Masashi Sugiyama.
\newblock Is the performance of my deep network too good to be true? a direct approach to estimating the bayes error in binary classification.
\newblock In {\em The Eleventh International Conference on Learning Representations}, 2023.

\bibitem{kingma2018glow}
Durk~P Kingma and Prafulla Dhariwal.
\newblock Glow: Generative flow with invertible 1x1 convolutions.
\newblock In S.~Bengio, H.~Wallach, H.~Larochelle, K.~Grauman, N.~Cesa-Bianchi, and R.~Garnett, editors, {\em Advances in Neural Information Processing Systems}, volume~31. Curran Associates, Inc., 2018.

\bibitem{krizhevsky2009learning}
Alex Krizhevsky, Geoffrey Hinton, et~al.
\newblock Learning multiple layers of features from tiny images.
\newblock 2009.

\bibitem{kurakin2018adversarial}
Alexey Kurakin, Ian~J. Goodfellow, and Samy Bengio.
\newblock Adversarial examples in the physical world.
\newblock In {\em 5th International Conference on Learning Representations, {ICLR} 2017, Toulon, France, April 24-26, 2017, Workshop Track Proceedings}. OpenReview.net, 2017.

\bibitem{kurakin2016adversarial}
Alexey Kurakin, Ian~J. Goodfellow, and Samy Bengio.
\newblock Adversarial machine learning at scale.
\newblock In {\em 5th International Conference on Learning Representations, {ICLR} 2017, Toulon, France, April 24-26, 2017, Conference Track Proceedings}. OpenReview.net, 2017.

\bibitem{lecunmnist}
Yann LeCun, Corinna Cortes, and Chris Burges.

\bibitem{li2023sok}
Linyi Li, Tao Xie, and Bo~Li.
\newblock Sok: Certified robustness for deep neural networks.
\newblock In {\em 44th {IEEE} Symposium on Security and Privacy, {SP} 2023, San Francisco, CA, USA, 22-26 May 2023}. IEEE, 2023.

\bibitem{li2022towards}
Renjue Li, Pengfei Yang, Cheng-Chao Huang, Youcheng Sun, Bai Xue, and Lijun Zhang.
\newblock Towards practical robustness analysis for dnns based on pac-model learning.
\newblock In {\em Proceedings of the 44th International Conference on Software Engineering}, ICSE '22, page 2189–2201, New York, NY, USA, 2022. Association for Computing Machinery.

\bibitem{lin2019robustness}
Wang Lin, Zhengfeng Yang, Xin Chen, Qingye Zhao, Xiangkun Li, Zhiming Liu, and Jifeng He.
\newblock Robustness verification of classification deep neural networks via linear programming.
\newblock In {\em Proceedings of the IEEE/CVF Conference on Computer Vision and Pattern Recognition (CVPR)}, pages 11418--11427, June 2019.

\bibitem{liu2019adaptiveface}
Hao Liu, Xiangyu Zhu, Zhen Lei, and Stan~Z. Li.
\newblock Adaptiveface: Adaptive margin and sampling for face recognition.
\newblock In {\em Proceedings of the IEEE/CVF Conference on Computer Vision and Pattern Recognition (CVPR)}, June 2019.

\bibitem{ma2018characterizing}
Xingjun Ma, Bo~Li, Yisen Wang, Sarah~M. Erfani, Sudanthi N.~R. Wijewickrema, Grant Schoenebeck, Dawn Song, Michael~E. Houle, and James Bailey.
\newblock Characterizing adversarial subspaces using local intrinsic dimensionality.
\newblock In {\em 6th International Conference on Learning Representations, {ICLR} 2018, Vancouver, BC, Canada, April 30 - May 3, 2018, Conference Track Proceedings}. OpenReview.net, 2018.

\bibitem{madry2017towards}
Aleksander Madry, Aleksandar Makelov, Ludwig Schmidt, Dimitris Tsipras, and Adrian Vladu.
\newblock Towards deep learning models resistant to adversarial attacks.
\newblock In {\em 6th International Conference on Learning Representations, {ICLR} 2018, Vancouver, BC, Canada, April 30 - May 3, 2018, Conference Track Proceedings}. OpenReview.net, 2018.

\bibitem{moon2015meta}
Kevin~R. Moon, Alfred~O. Hero, and B.~Véronique Delouille.
\newblock Meta learning of bounds on the bayes classifier error.
\newblock In {\em 2015 IEEE Signal Processing and Signal Processing Education Workshop (SP/SPE)}, pages 13--18, 2015.

\bibitem{muller2022certified}
Mark~Niklas M{\"{u}}ller, Franziska Eckert, Marc Fischer, and Martin~T. Vechev.
\newblock Certified training: Small boxes are all you need.
\newblock {\em CoRR}, abs/2210.04871, 2022.

\bibitem{nielsen2014generalized}
Frank Nielsen.
\newblock Generalized bhattacharyya and chernoff upper bounds on bayes error using quasi-arithmetic means.
\newblock {\em Pattern Recognition Letters}, 42:25--34, 2014.

\bibitem{noshad2019learning}
Morteza Noshad, Li~Xu, and Alfred Hero.
\newblock Learning to benchmark: Determining best achievable misclassification error from training data.
\newblock {\em arXiv preprint arXiv:1909.07192}, 2019.

\bibitem{scikit-learn}
F.~Pedregosa, G.~Varoquaux, A.~Gramfort, V.~Michel, B.~Thirion, O.~Grisel, M.~Blondel, P.~Prettenhofer, R.~Weiss, V.~Dubourg, J.~Vanderplas, A.~Passos, D.~Cournapeau, M.~Brucher, M.~Perrot, and E.~Duchesnay.
\newblock Scikit-learn: Machine learning in {P}ython.
\newblock {\em Journal of Machine Learning Research}, 12:2825--2830, 2011.

\bibitem{peterson2019human}
Joshua~C. Peterson, Ruairidh~M. Battleday, Thomas~L. Griffiths, and Olga Russakovsky.
\newblock Human uncertainty makes classification more robust.
\newblock In {\em Proceedings of the IEEE/CVF International Conference on Computer Vision (ICCV)}, October 2019.

\bibitem{renggli2021evaluating}
C{\'{e}}dric Renggli, Luka Rimanic, Nora Hollenstein, and Ce~Zhang.
\newblock Evaluating bayes error estimators on real-world datasets with feebee.
\newblock In Joaquin Vanschoren and Sai{-}Kit Yeung, editors, {\em Proceedings of the Neural Information Processing Systems Track on Datasets and Benchmarks 1, NeurIPS Datasets and Benchmarks 2021, December 2021, virtual}, 2021.

\bibitem{ripley1996pattern}
Brian~D. Ripley.
\newblock {\em Pattern Recognition and Neural Networks}.
\newblock Cambridge University Press, 1996.

\bibitem{robey2022probabilistically}
Alexander Robey, Luiz Chamon, George~J. Pappas, and Hamed Hassani.
\newblock Probabilistically robust learning: Balancing average and worst-case performance.
\newblock In Kamalika Chaudhuri, Stefanie Jegelka, Le~Song, Csaba Szepesvari, Gang Niu, and Sivan Sabato, editors, {\em Proceedings of the 39th International Conference on Machine Learning}, volume 162 of {\em Proceedings of Machine Learning Research}, pages 18667--18686. PMLR, 17--23 Jul 2022.

\bibitem{sekeh2020learning}
Salimeh~Yasaei Sekeh, Brandon Oselio, and Alfred~O. Hero.
\newblock Learning to bound the multi-class bayes error.
\newblock {\em IEEE Transactions on Signal Processing}, 68:3793--3807, 2020.

\bibitem{sharif2016accessorize}
Mahmood Sharif, Sruti Bhagavatula, Lujo Bauer, and Michael~K. Reiter.
\newblock Accessorize to a crime: Real and stealthy attacks on state-of-the-art face recognition.
\newblock In {\em Proceedings of the 2016 ACM SIGSAC Conference on Computer and Communications Security}, CCS '16, pages 1528--1540, New York, NY, USA, 2016. Association for Computing Machinery.

\bibitem{shi2021fast}
Zhouxing Shi, Yihan Wang, Huan Zhang, Jinfeng Yi, and Cho-Jui Hsieh.
\newblock Fast certified robust training with short warmup.
\newblock In M.~Ranzato, A.~Beygelzimer, Y.~Dauphin, P.S. Liang, and J.~Wortman Vaughan, editors, {\em Advances in Neural Information Processing Systems}, volume~34, pages 18335--18349. Curran Associates, Inc., 2021.

\bibitem{shorten2019survey}
Connor Shorten and Taghi~M. Khoshgoftaar.
\newblock A survey on image data augmentation for deep learning.
\newblock {\em Journal of Big Data}, 6(1):60, Jul 2019.

\bibitem{singh2019abstract}
Gagandeep Singh, Timon Gehr, Markus P\"{u}schel, and Martin Vechev.
\newblock An abstract domain for certifying neural networks.
\newblock {\em Proc. ACM Program. Lang.}, 3(POPL), jan 2019.

\bibitem{szegedy2013intriguing}
Christian Szegedy, Wojciech Zaremba, Ilya Sutskever, Joan Bruna, Dumitru Erhan, Ian~J. Goodfellow, and Rob Fergus.
\newblock Intriguing properties of neural networks.
\newblock In Yoshua Bengio and Yann LeCun, editors, {\em 2nd International Conference on Learning Representations, {ICLR} 2014, Banff, AB, Canada, April 14-16, 2014, Conference Track Proceedings}, 2014.

\bibitem{theisen2021evaluating}
Ryan Theisen, Huan Wang, Lav~R Varshney, Caiming Xiong, and Richard Socher.
\newblock Evaluating state-of-the-art classification models against bayes optimality.
\newblock In M.~Ranzato, A.~Beygelzimer, Y.~Dauphin, P.S. Liang, and J.~Wortman Vaughan, editors, {\em Advances in Neural Information Processing Systems}, volume~34, pages 9367--9377. Curran Associates, Inc., 2021.

\bibitem{tramer2020adaptive}
Florian Tramer, Nicholas Carlini, Wieland Brendel, and Aleksander Madry.
\newblock On adaptive attacks to adversarial example defenses.
\newblock In H.~Larochelle, M.~Ranzato, R.~Hadsell, M.F. Balcan, and H.~Lin, editors, {\em Advances in Neural Information Processing Systems}, volume~33, pages 1633--1645. Curran Associates, Inc., 2020.

\bibitem{wang2021robot}
Jingyi Wang, Jialuo Chen, Youcheng Sun, Xingjun Ma, Dongxia Wang, Jun Sun, and Peng Cheng.
\newblock Robot: Robustness-oriented testing for deep learning systems.
\newblock In {\em 2021 IEEE/ACM 43rd International Conference on Software Engineering (ICSE)}, pages 300--311, 2021.

\bibitem{wang2020improving}
Yisen Wang, Difan Zou, Jinfeng Yi, James Bailey, Xingjun Ma, and Quanquan Gu.
\newblock Improving adversarial robustness requires revisiting misclassified examples.
\newblock In {\em 8th International Conference on Learning Representations, {ICLR} 2020, Addis Ababa, Ethiopia, April 26-30, 2020}. OpenReview.net, 2020.

\bibitem{xiao2017fashion}
Han Xiao, Kashif Rasul, and Roland Vollgraf.
\newblock Fashion-mnist: a novel image dataset for benchmarking machine learning algorithms.
\newblock {\em CoRR}, abs/1708.07747, 2017.

\bibitem{xu2020automatic}
Kaidi Xu, Zhouxing Shi, Huan Zhang, Yihan Wang, Kai-Wei Chang, Minlie Huang, Bhavya Kailkhura, Xue Lin, and Cho-Jui Hsieh.
\newblock Automatic perturbation analysis for scalable certified robustness and beyond.
\newblock In H.~Larochelle, M.~Ranzato, R.~Hadsell, M.F. Balcan, and H.~Lin, editors, {\em Advances in Neural Information Processing Systems}, volume~33, pages 1129--1141. Curran Associates, Inc., 2020.

\bibitem{zhang2019theoretically}
Hongyang Zhang, Yaodong Yu, Jiantao Jiao, Eric Xing, Laurent~El Ghaoui, and Michael Jordan.
\newblock Theoretically principled trade-off between robustness and accuracy.
\newblock In Kamalika Chaudhuri and Ruslan Salakhutdinov, editors, {\em Proceedings of the 36th International Conference on Machine Learning}, volume~97 of {\em Proceedings of Machine Learning Research}, pages 7472--7482. PMLR, 09--15 Jun 2019.

\bibitem{zhang2019limitations}
Huan Zhang, Hongge Chen, Zhao Song, Duane~S. Boning, Inderjit~S. Dhillon, and Cho{-}Jui Hsieh.
\newblock The limitations of adversarial training and the blind-spot attack.
\newblock In {\em 7th International Conference on Learning Representations, {ICLR} 2019, New Orleans, LA, USA, May 6-9, 2019}. OpenReview.net, 2019.

\bibitem{zhang2018efficient}
Huan Zhang, Tsui-Wei Weng, Pin-Yu Chen, Cho-Jui Hsieh, and Luca Daniel.
\newblock Efficient neural network robustness certification with general activation functions.
\newblock In {\em Proceedings of the 32nd International Conference on Neural Information Processing Systems}, volume~31 of {\em NIPS'18}, page 4944–4953, Red Hook, NY, USA, 2018. Curran Associates Inc.

\bibitem{zhang2023coophance}
Quan Zhang, Yongqiang Tian, Yifeng Ding, Shanshan Li, Chengnian Sun, Yu~Jiang, and Jiaguang Sun.
\newblock Coophance: Cooperative enhancement for robustness of deep learning systems.
\newblock In {\em Proceedings of the 32nd ACM SIGSOFT International Symposium on Software Testing and Analysis}, ISSTA 2023, page 753–765, New York, NY, USA, 2023. Association for Computing Machinery.

\bibitem{zhang2024certified}
Ruihan Zhang and Jun Sun.
\newblock Certified robust accuracy of neural networks are bounded due to bayes errors.
\newblock In {\em Computer Aided Verification}. Springer International Publishing, 2024.

\bibitem{zhang2023proa}
Tianle Zhang, Wenjie Ruan, and Jonathan~E. Fieldsend.
\newblock Proa: A probabilistic robustness assessment against functional perturbations.
\newblock In Massih-Reza Amini, St{\'e}phane Canu, Asja Fischer, Tias Guns, Petra Kralj~Novak, and Grigorios Tsoumakas, editors, {\em Machine Learning and Knowledge Discovery in Databases}, pages 154--170, Cham, 2023. Springer Nature Switzerland.

\end{thebibliography}
\end{document}